\documentclass{article}

\usepackage{arxiv}

\usepackage[utf8]{inputenc} % allow utf-8 input
\usepackage[T1]{fontenc}    % use 8-bit T1 fonts
\usepackage{hyperref}       % hyperlinks
\hypersetup{
  colorlinks,
  citecolor=Violet,
  linkcolor=Red,
  urlcolor=Blue}
\usepackage{url}            % simple URL typesetting
\usepackage{booktabs}       % professional-quality tables
\usepackage{amsfonts}       % blackboard math symbols
\usepackage{nicefrac}       % compact symbols for 1/2, etc.
\usepackage{microtype}      % microtypography
\usepackage{lipsum}

\usepackage{amssymb}
\usepackage{amsmath}
\usepackage{graphicx}
\usepackage{caption}
\usepackage[dvipsnames]{xcolor}
\usepackage{subcaption}
\graphicspath{ {images/} }
\usepackage{float}
\usepackage{amsthm}
\usepackage{listings}
\usepackage{algorithm,algcompatible}
\makeatletter
\def\munderbar#1{\underline{\sbox\tw@{$#1$}\dp\tw@\z@\box\tw@}}
\makeatother
\newcommand{\norm}[1]{\left\lVert#1\right\rVert}
\usepackage{algorithm}
\usepackage{algorithmicx}
\usepackage{algpseudocode}
\usepackage{makecell}
\usepackage{booktabs}       % professional-quality tables
\usepackage{url}
\usepackage{multirow}
\usepackage{lipsum}
\usepackage{bbm}

\newtheorem{theorem}{Theorem}

\newtheorem{lemma}{Lemma}
\newtheorem{problem}{Problem}
\newtheorem{remark}{Remark}

\newtheorem{definition}{Definition}

\newtheorem{assumption}{Assumption}

\newcommand{\R}{\mathbb{R}}
\newcommand{\Z}{\mathbb{Z}}

\newcommand{\uu}{\mathbf{u}}
\newcommand{\s}{\mathbf{s}}

\newcommand{\sig}{\mathbf{\sigma}}

\newcommand{\Ss}{\mathcal{S}}
\newcommand{\W}{\mathcal{W}}

\title{Towards Guaranteed Safety Assurance of Automated Driving Systems with Scenario Sampling: An Invariant Set Perspective (Extended Version)}

\author{Bowen Weng, Linda Capito,  Umit Ozguner, Keith Redmill% <-this % stops a space
\thanks{Bowen Weng (\texttt{weng.172@osu.edu}), Linda Capito (\texttt{capitoruiz.1@osu.edu}), Umit Ozguner (\texttt{ozguner.1@osu.edu}), and Keith Redmill (\texttt{redmill.1@osu.edu}) are with Department of Electrical and Computer Engineering at Ohio State University, OH, USA.}
}

\begin{document}
\maketitle

\begin{abstract}
How many scenarios are sufficient to validate the safe Operational Design Domain (ODD) of an Automated Driving System (ADS) equipped vehicle? Is a more significant number of sampled scenarios guaranteeing a more accurate safety assessment of the ADS? Despite the various empirical success of ADS safety evaluation with scenario sampling in practice, some of the fundamental properties are largely unknown. 
This paper seeks to remedy this gap by formulating and tackling the scenario sampling safety assurance problem from a set invariance perspective. 
First, a novel conceptual equivalence is drawn between the \emph{scenario sampling safety assurance problem} and the \emph{data-driven robustly controlled forward invariant set validation and quantification problem}. This paper then provides a series of resolution complete and probabilistic complete solutions with finite-sampling analyses for the safety validation problem that authenticates a given ODD. On the other hand, the quantification problem escalates the validation challenge and starts looking for a safe sub-domain of a particular property. This inspires various algorithms that are provably probabilistic incomplete, probabilistic complete but sub-optimal, and asymptotically optimal. Finally, the proposed asymptotically optimal scenario sampling safety quantification algorithm is also empirically demonstrated through simulation experiments.
\end{abstract}

% keywords can be removed

\newpage
\section{Introduction}\label{sec:introduction}

If an Original Equipment Manufacturer (OEM) builds a product such as an Automated Driving System (ADS) or Advanced Driver Assistance System (ADAS) equipped vehicle, it is also expected that the OEM provides the Operational Design Domain (ODD) to fully specify the vehicle's capability. The safety assurance problem naturally arises as one seeks to either derive such an ODD (i.e., \emph{quantification}) or verify the validity of a given ODD (i.e., \emph{validation}). Since the constructive modeling of ADS/ADAS-equipped vehicles and real-world traffic interactions are complex, the data-driven scenario sampling approach is an important alternative for safety assurance. In general, a scenario sampling algorithm assesses the safety performance of the Subject Vehicle (SV) by (i) specifying $N$ initial configurations and (ii) starting from each configuration observing a trajectory of states by executing a certain feedback control policy for the controllable factors. In this paper, we consider two classes of scenario sampling safety assurance problems, (i) the safety validation problem and (ii) the safety quantification problem.

Informally speaking, the \emph{scenario sampling safety validation problem} seeks to validate the given safe operable domain (or intended functions) of the ADS/ADAS-equipped vehicle. This is structurally studied under the so-called Safety of the Intended Functionality (SOTIF) as part of ISO standards~\cite{kirovskii2019driver}. Moreover, Riedmaier~et al.~\cite{riedmaier2020survey} has classified the scenario sampling safety validation algorithms into two categories, the testing-based and the falsification-based methods~\cite{corso2020survey}. The testing-based methods seek sufficient coverage of all domain-admitted scenarios, whereas the falsification-based methods focus on finding corner-case incidents. Suppose the given operational domain is not valid. The falsification-based validation algorithm is typically more efficient than the testing-based method in validating the underlying safety property, given its biased focus on safety-critical events. Typical falsification-based methods include solutions using expert knowledge~\cite{bagschik2018ontology}, model insights~\cite{capito2020modeled, klischat2020scenario}, adaptive sampling schemes~\cite{feng2020adaptive, feng2020testing}, and various learning-based frameworks~\cite{chen2020adversarial, ding2020learning}. On the other hand, if the given domain is indeed sufficiently safe, both methods will require a certain number of samples to claim the safety property with a high confidence level. This inspires some work in testing scenario library generation~\cite{bagschik2018ontology, feng2020testing}. However, the scenario defined in those studies is typically based on abstracted features with concrete scenario designs, which may not capture the complex dynamic propagation of traffic scenes. This further leads to various problems as discussed in~\cite{hauer2020re, capito2020modeled}. In the on-road testing regime, Karla~et al. provide a finite-mile guarantee with a high confidence level that assures the SV having a certain fatality rate~\cite{kalra2016driving}. However, to the best of our knowledge, it is still unclear how the number of samples quantitatively relates to the scenario sampling methods' safety validation outcome. 

Furthermore, simply validating the given claim with a Boolean-type answer is not always sufficient. Especially when the claimed safe operable domain is invalid, one may also want to explore various scenarios and identify possible sub-domains that could still induce the described safety property. This is referred to as the \emph{scenario sampling safety quantification problem}, which potentially also comes with the testing-based and the falsification-based methods. However, establishing a quantification algorithm that provably finds a particular safe operable domain is not easy. Some seemingly working solutions could still fail even if one samples infinitely many scenarios, as we will show in Section~\ref{sec:quantification}. The outcome of the quantification algorithm is of practical value in comparing the performance of different ADS/ADAS-equipped vehicles, and can also be applied to the issues raised by the Object and Event Detection and Response (ODER)~\cite{thorn2018framework}. 
 
To help address the aforementioned problems, the scenario sampling methods have been implemented in computer simulations~\cite{kluck2018using}, augmented reality test environment~\cite{feng2021intelligent}, and real-world experiments~\cite{forkenbrock2015nhtsa} through various techniques such as the expert knowledge based scenario library~\cite{putz2017system, althoff2017commonroad}, the naturalistic behavioral data driven scenario extraction~\cite{feng2021intelligent}, the adaptive stressing test~\cite{koren2018adaptive, corso2019adaptive}, and the backtracking process algorithm~\cite{hejase2020methodology}. To a certain extent, all the solutions above ``assess" the safety performance. However, even in the ideal simulation environment capable of running a large number of scenarios and initializing scenarios at arbitrary configurations, some fundamental properties of the scenario sampling safety assurance problem remain unclear. For example:
\begin{enumerate}
    \item How many scenarios are sufficient to validate the claimed ODD?
    \item Is one guaranteed to find at least one safe sub-domain as the number of sampled scenarios tends to infinity?
    \item Is one guaranteed to find all safe sub-domains (if applicable) as the number of sampled scenarios tends to infinity?
\end{enumerate}
To the best of our knowledge, there is no rigorous answer to any of the questions above in the scenario sampling literature. These questions inspire our work, and the proposed solution addresses the presented problems and questions. Overall, the contributions of this paper are listed as follows.

\textbf{Problem formulation from the set invariance perspective}: For the first time, the scenario sampling safety assurance problems presented above are shown to be related to a class of data-driven set invariance validation and quantification problems. Such a novel problem formulation allows us to obtain the statistical safety assurance of ADS/ADAS-driven vehicles in a provably correct and empirically effective manner.

\textbf{The safety validation problem}: For the safety validation problem, we first study a group of sampling algorithms with completeness guarantees (i.e., the probability of solving the validation problem is either one or tends to one as a sufficient number of scenarios are sampled). Such a sampling ``sufficiency" is then quantified rigorously for various algorithms that are either resolution complete or probabilistic complete~\cite{karaman2011sampling}. Ultimately, the finite-sampling guarantee is provided to statistically validate the controlled forward invariant set, interchangeably, the provided ODD's authenticity for the ADS/ADAS equipped vehicle. 

\textbf{The safety quantification problem}: Inspired by the sampling based motion planning work~\cite{karaman2011sampling}, we divide the safety quantification problem into two sub-problems, (i) the feasible safety quantification problem and (ii) the optimal safety quantification problem. The corresponding solutions are then known as the probabilistic complete algorithms (i.e., the probability of finding a safe sub-domain tends to one as the number of sampled scenarios tends to infinity) and the asymptotically optimal algorithms (i.e., the probability of finding the optimal safe sub-domains tends to one as the number of samples tends to infinity), respectively. Furthermore, we present a series of algorithm examples that are (i) intuitively effective but probabilistic incomplete, (ii) probabilistic complete but sub-optimal, and (iii) theoretically guaranteed asymptotically optimal . The proposed asymptotically optimal  scenario sampling algorithm's effectiveness is demonstrated through black-box testing in high-fidelity simulations.

Note that the set invariance and the safety concept are closely related in principle. Many safety work have the set invariance idea embedded in the proposal, including the reachability analysis (RS)~\cite{mitchell2007comparing, althoff2014online, park2012model} and the control barrier function (CBF) related methods~\cite{ames2016control, chen2017obstacle, tuncali2018reasoning}. The proposed framework in this paper enhances the previous efforts in data-driven set invariance validation~\cite{wang2020scenario} and system stability in probability~\cite{tsien1954engineering-probstability, dellnitz1999approximation}, and makes case-specific improvements for the particular problem of ADS safety assurance. Although both backward reachability analysis (BRS) and our proposed solution derive a certain forward invariant set which serves as the safe set (or its complement), BRS is a ``model-based" method relying on the explicit solution of the Hamilton-Jacobi-Bellman (HJB) equation to obtain the set{~\cite{mitchell2007comparing}}. Our proposed method takes a ``sampling-based" approach in approximating the robustly controlled forward invariant set. Such a methodological difference also makes the proposed method compatible with more complex problems such as the non-Lipschitz dynamics and the non-convex set of failure events. On the other hand, the forward reachability analysis (FRS){~\cite{althoff2014online}} is primarily focusing on developing local solutions for any given state, hence the forward reachable set does not necessarily induce the overall safety property{~\cite{mitchell2007comparing}}. Furthermore, the aforementioned methods (BRS, FRS, and CBF) typically take advantage of the invariant set to derive the safe controller, while the proposed method is mainly concerned with the testing and evaluation of an existing system.
Finally, it is also worth mentioning that Nahhal et al.~\cite{nahhal2007test} cast the safety validation problem as a path planning problem, which is also related to the safety quantification problem formulation presented in Section~\ref{sec:quantification}. However, the path planning perspective considered by Nahhal et al. primarily focuses on searching for failures from the falsification perspective. Such a sampled observation fails to capture the overall safety performance of the presented ADS. Meanwhile, the proposed safety quantification problem (see Problem~\ref{prob:feasible-qnt} and Problem~\ref{prob:optimal-qnt} in Section~\ref{sec:prelimilaries}) comes with more complete safety property characterizations.

% especially the \emph{``$\epsilon\delta$-almost" robustly controlled forward invariant set} (Definition~\ref{def:almost-cpis} in Section~\ref{sec:validation}), not only enhances the previous efforts in data-driven set invariance validation~\cite{dellnitz1999approximation, wang2020scenario}, but also gives rise to a variety of statistical safety assessment approaches. It could be of independent interest for future work beyond the safety validation and quantification problems for ADS/ADAS-equipped vehicles studied by this paper.

\begin{figure*}[!t]
    \centering
    \includegraphics[width=0.95\linewidth]{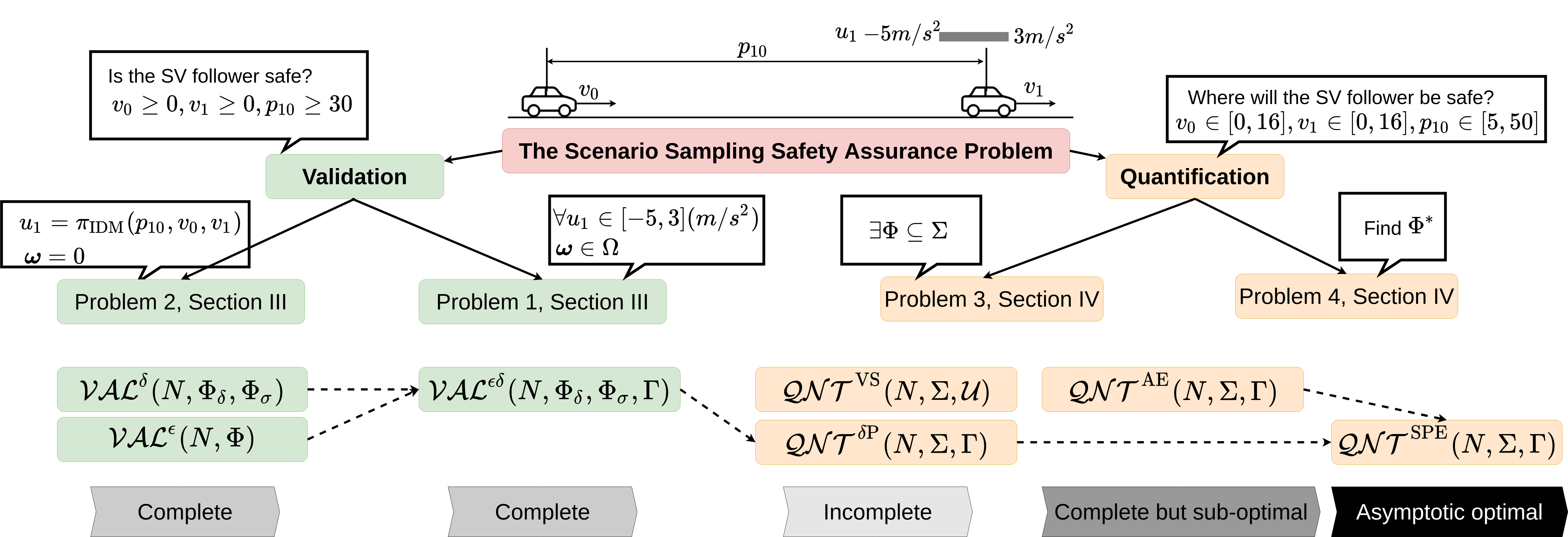}
    \caption{An overview of the formulated problems and presented algorithms in Section~\ref{sec:validation} and Section~\ref{sec:quantification}. Consider the lead-vehicle following case with the potential robustly controlled forward invariant set (i.e., the safe ODD) defined over the three-dimensional state space. (i) The scenario sampling safety assurance problem in this case can be further divided into a validation problem and a quantification problem. (ii) The validation problem asks questions such as ``Is the SV follower safe for initial following distance greater than 30 meters, if the lead vehicle follows the IDM based policy and the uncertainties are zero?" (Problem~\ref{prob:val_dp}), and ``Is the SV follower safe for initial following distance greater than 5 meters assuming an explicit lead vehicle policy with its longitudinal acceleration capability being subject to a known saturation?" (Problem~\ref{prob:val}). (iii) The quantification problem raises questions such as ``Can you find a safe following distance $p$ such that for all possible velocities for both vehicles the SV follower remains collision-free?" (Problem~\ref{prob:feasible-qnt}), and `Can you find the minimum safe following distance $p^*$ such that for all possible velocities for both vehicles the SV follower remains collision-free?" (Problem~\ref{prob:optimal-qnt}). (iv) A series of algorithms of various properties are presented for both problems. The validation algorithms and the quantification algorithms are closely related with relative dependencies as we will also address in Remark~\ref{rmk:unify_val_and_qnt}.}
    \label{fig:overview}
\end{figure*}

An overview of the problems and algorithms presented and analyzed in this paper is shown in Figure~\ref{fig:overview}. The construction of this paper is as follows. Section~\ref{sec:prelimilaries} revisits the set invariance concepts from the control literature, formulates the scenario definition from the dynamic system perspective, and establishes the safety definition from the set invariance perspective. This further leads to two classes of scenario sampling safety assurance problems. The safety validation problem is detailed in Section~\ref{sec:validation}. The safety quantification problem discussed in Section~\ref{sec:quantification} partially builds upon the content in Section~\ref{sec:validation}, but further escalates the challenges with theoretically guaranteed and empirically validated solutions provided. Finally, Section~\ref{sec:conclusion} summarizes the paper.

\begin{remark}
    None of the presented techniques in this paper are prohibiting the algorithms from working with high-dimensional system with high-cardinality sets. The particular case shown in Figure~\ref{fig:overview} and also studied in Section~\ref{sec:quantification} is only for proof of concept. 
\end{remark}

\textbf{Notations: } The set of real and positive real numbers are denoted by $\R$ and $\R_{>0}$ respectively. $\Z$ denotes the set of all positive integers and $\Z_N=\{1,\ldots,N\}$. The $\ell_{\infty}$-norm is denoted by $\norm{\cdot}$. $|\mathcal{X}|$ is the cardinality of the set $\mathcal{X}$. The function $d(\mathbf{p}, \mathcal{X})$ denotes the point-to-set distance between a point $\mathbf{p} \in \R^n$ and a set $\mathcal{X} \subseteq \R^n$ as $d(\mathbf{p}, \mathcal{X}) = \underset{\mathbf{x}\in\mathcal{X}}{\min}\norm{\mathbf{p}-\mathbf{x}}$. Let $\partial\mathcal{X}$ be the boundary of $\mathcal{X}$. Consider the signed point-to-set distance as
\begin{align*}
    d_s(\mathbf{p}, \mathcal{X}) = \begin{cases}
    d(\mathbf{p}, \partial\mathcal{X}), & \text{if }\mathbf{p} \notin \mathcal{X}\\
    -d(\mathbf{p}, \partial\mathcal{X}), & \text{otherwise. }
    \end{cases}
\end{align*}
A continuous function $\alpha:[0,a)\rightarrow[0,\infty)$ for some $a>0$ is a class $\mathcal{K}$ function if it is strictly increasing and $\alpha(0)=0$.

\section{Preliminaries} \label{sec:prelimilaries}

\subsection{Revisit set invariance}
Consider a general nonlinear system:
\begin{equation}\label{eq:s-dyn}
    \s(t+1) = f(\s(t); \boldsymbol\omega(t)),
\end{equation}
where the state $\s \in \Ss \subseteq \R^n$, the disturbance $\boldsymbol\omega \in \W \subseteq \R^w$. Both $\Ss$ and $\W$ are compact. The time step $t \in \Z$ and transition function $f:\Ss \times \W \rightarrow \Ss$ is Lipschitz continuous.
\begin{definition}\label{def:robust-fi}
    \textbf{(Robustly forward invariant set)} The set $\Phi \subseteq \Ss$ is robustly forward invariant for the system~\eqref{eq:s-dyn} if for all $\s(0) \in \Phi$ and all $\boldsymbol\omega(t)\in\W$ the solution satisfies $\s(t) \in \Phi, \forall t > 0$.
\end{definition}
% The above definition induces a property that, if a subset of $\Ss$ contains a state at time $t$, then it will consistently contain it in the future. 
Furthermore, by adding control capability to the system~\eqref{eq:s-dyn}, we have the nonlinear control system as
\begin{equation}\label{eq:su-dyn}
    \s(t+1) = f(\s(t), \uu(t); \boldsymbol\omega(t)),
\end{equation}
where the action $\uu \in \mathcal{U}$ and transition function is modified as $f: \Ss \times \mathcal{U} \times \W \rightarrow \Ss$.
\begin{definition}\label{def:robust-cfi}
    \textbf{(Robustly controlled forward invariant set)} The set $\Phi \subseteq \Ss$ is robustly controlled forward invariant for the system~\eqref{eq:su-dyn} if for all $\s(0) \in \Phi$, all $\boldsymbol\omega(t)\in\W$, there exists a nonempty set $ \Gamma := \Gamma(\boldsymbol\sig(t)) \subseteq \mathcal{U}$, such that for all $\uu(t) \in \Gamma$ the solution of~\eqref{eq:su-dyn} satisfies $\s(t) \in \Phi, \forall t > 0$.
\end{definition}
\begin{remark}\label{rmk:Gamma_t}
    The above definition is more general than a typical controlled forward invariant set definition~\cite{blanchini1999set}, where the latter only requires the existence of a single control action $\uu(t)\in\mathcal{U}$. Furthermore, as indicated by definition, the admissible action space $\Gamma(\boldsymbol\sig(t))$ is time-variant and state-dependent in general. In the remainder of this paper, we mostly consider the time-invariant $\Gamma$ unless specified otherwise.
\end{remark}

% \begin{remark}
%     For both Definition~\ref{def:robust-fi} and Definition~\ref{def:robust-cfi}, the word ``robustly" can be ignored if the system~\eqref{eq:s-dyn} and~\eqref{eq:su-dyn} are with zero uncertainties, i.e., $\boldsymbol\omega(t)=0, \forall t$.
% \end{remark}

\subsection{The scenario system}
In this paper, we consider a similar definition of a scenario to~\cite{capito2020modeled} where we define the testing scenario as a dynamic system subject to the motion equation of~\eqref{eq:su-dyn}. With a little abuse of notation, we formally present the scenario dynamics as:
\begin{equation}\label{eq:snapshot-dyn}
    \boldsymbol\sig(t+1) = f(\boldsymbol\sig(t), \uu(t); \boldsymbol\omega(t)).
\end{equation}
The state $\boldsymbol\sig \in \Sigma \subseteq \R^n$ denotes all dynamic and static state properties of all participants. This not only includes the typical dynamic states of the traffic participants (e.g., position and velocity of all vehicles), but could also involve the road-level descriptions (e.g. surface friction), the traffic infrastructures (e.g. lightning condition), and the environment factors (e.g., weather). Correspondingly, $\uu \in \mathcal{U}$ represents all actions capable of controlling the states above, with only one exception being the SV control. In the testing scenario context, the SV is often treated as a black-box system, with only part of its states being observable. Finally, $\boldsymbol\omega(t) \in \W$ includes system disturbances and potential uncertainties, which also admits the following assumption.
\begin{assumption}
    Assume $\boldsymbol\omega(t) \in \W$ is uniformly bounded, i.e., there exists $\bar{\omega} > 0$ such that $\norm{\boldsymbol\omega(t)} \leq \bar{\omega}, \forall t \in \Z$. 
\end{assumption}
The above property is well-adopted in the robust control literature and it remains valid for many commonly observed factors in the ADS applications such as the tire friction, air resistance, and road gradients.
\begin{remark}\label{rmk:underactuation}
    Note that the absence of SV control action in $\uu(t)$ makes the testing scenario conceptually close to an \emph{``underactuated"} system. While the particular underlying SV policy remains unknown, its action is often uniformly bounded. Hence one can also treat it as part of the disturbance $\boldsymbol\omega(t)$. 
\end{remark}

\begin{remark}
    The control action $\uu$ and the disturbance $\boldsymbol\omega$ are interchangeable. For example, the tire friction can be perfectly controlled in computer simulators, but it is commonly considered as a source of uncertainty in real-world tests. 
\end{remark}

Within the intelligent vehicle context, various scenario definitions have been proposed. Ulbrich~et al.~\cite{ulbrich2015defining} define the scenario as a temporal sequence of scene elements, with actions and events of the participating elements occurring within this sequence. As an extension, Menzel~et al.~\cite{menzel2018scenarios} further classifies the scenario into three abstractive levels, i.e., the functional, the logical, and the concrete scenarios. While the abstracted scenarios are mostly aligning with natural intuitions, the simplification may not be sufficient for the safety assurance purpose as pointed out in~\cite{hauer2020re, capito2020modeled}. On the other hand, Capito~et al. formulate the scenario as an automaton (Definition 1 in~\cite{capito2020modeled}). With the formal language of the automaton theory, they further define the ``run of a scenario" being a chronological sequence of snapshots of traffic states. As a result, the scenario design problem then seeks to establish a feedback controller for all traffic participants other than the Subject Vehicle (SV) (where the SV is treated as a black-box system in general). This is fundamentally equivalent to the scenario definition provided in this paper, and the run of a scenario in~\cite{capito2020modeled} is essentially a trajectory of state points propagated through the dynamics~\eqref{eq:snapshot-dyn} supplied with an initialization condition $\boldsymbol\sigma(0)$ and a certain feedback control policy $\uu(t)=\pi(\boldsymbol\sig(t)), \pi: \Sigma\rightarrow \mathcal{U}$.

The methodology derived in this paper is primarily based on the scenario dynamics~\eqref{eq:snapshot-dyn}. It is still transferable to the abstracted scenario models with less complex designs, yet details are out of the scope of this paper.

\subsection{Safety assurance with scenario sampling}
Combing the set invariance definitions with the scenario dynamics~\eqref{eq:snapshot-dyn}, we formally present the safety definition as follows.
\begin{definition}\label{def:safety}
    \textbf{(Safety)} The subject vehicle is safe in $\Phi \subseteq \Sigma$ for the system~\eqref{eq:snapshot-dyn} if $\Phi$ is robustly controlled forward invariant for~\eqref{eq:snapshot-dyn}.
\end{definition}

\begin{remark}
    Note that the set invariance property is only a necessary condition for safety. Under certain conditions, the controlled forward invariance property could also apply to unsafety (e.g. the set of inevitable collisions).
\end{remark}

It is straightforward that the Definition~\ref{def:safety} aligns with the common sense of safety. For example, in the functional safety domain, safety is defined as \emph{the freedom from unacceptable risk of physical injury, of damage to the health of people and the environment}~\cite{kirovskii2019driver}. This is essentially equivalent to Definition~\ref{def:safety} if one considers $\Phi$ as the complement of all events with unacceptable risk.

Rigorous safety performance assessment of the SV with the dynamic propagation of~\eqref{eq:snapshot-dyn} is challenging. With falsification-based methods, one can only derive counter-examples and does not obtain the complete safety performance profile. On the other hand, the testing-based solution also suffers from two problems. First, existing work lacks rigorous quantification of ``the sufficient coverage" of all state-action pairs. Moreover, simply visiting a state-action pair once is not necessarily enough due to the ``underactuation" addressed by Remark~\ref{rmk:underactuation} and various uncertainties dictated by $\boldsymbol\omega$ in~\eqref{eq:snapshot-dyn}. Similar observations are also illustrated in Figure~\ref{fig:lead_follow_demo} and identified by Hauer~et al.~\cite{hauer2020re} from the regression test perspective. As a result, most existing work seek to assess safety with ``abstracted" scenarios where a finite-dimensional feature vector is adopted to simplify the scenario description. However, the lack of analytical sampling analysis for the scenario sampling algorithm remains a problem for scenarios defined in lower-dimensions and more complex sample space. Even if one can arrive at a certain theoretically guaranteed property w.r.t. the simplified scenario definition, the analyses remain questionable for the more general scenario dynamics. 

The proposed framework in this paper directly considers the general scenario propagation as defined by~\eqref{eq:snapshot-dyn}. A typical safety assurance algorithm with scenario sampling always takes the form of sampling $N$ runs of scenarios. A run of a scenario is then formally described through Algorithm~\ref{alg:run_scene} along with Assumption~\ref{asp:run_scene}. 

\begin{algorithm}[H]
    \begin{algorithmic}[1]
    \State {\bf Input:} The initial state $\boldsymbol\sigma_0$, the trajectory horizon $K$, the feedback control policy $\pi: \Sigma\rightarrow \mathcal{U}$ for the scenario propagation of~\eqref{eq:snapshot-dyn}.
    \State {\bf Initialize: } $i=1$, $\boldsymbol\sig=\boldsymbol\sig_0$, trajectory $\mathcal{T}=\{\boldsymbol\sig_0\}$.
    \State {{\bf While} $i < K$:}
    \State {\ \ \ \ $\uu = \pi(\boldsymbol\sig)$}
    \State {\ \ \ \ get $\boldsymbol\sigma'$ from simulating dynamics~\eqref{eq:snapshot-dyn}}
    \State {\ \ \ \ $\mathcal{T}.\text{append}(\boldsymbol\sig'$)}
    \State {\ \ \ \ $\boldsymbol\sig=\boldsymbol\sig'$}
    \State {\ \ \ \ $i$+=$1$}
    \State {{\bf Output:} $\mathcal{T}$}
    \end{algorithmic}
    \caption{Run of a scenario $\mathcal{RS}(\boldsymbol\sigma_0, K, \pi(\cdot))$} \label{alg:run_scene}
\end{algorithm}
As the focus of this paper is the scenario sampling approach, which primarily appears in computer simulators. The following assumption is standard.
\begin{assumption}\label{asp:run_scene}
    Assume Algorithm~\ref{alg:run_scene} takes arbitrary $\boldsymbol\sigma_0\in\Sigma$ and $K\in[2,\infty)$, i.e., one can ``reset" the run of a scenario to an arbitrary initial state and each run of scenario returns a finite-length trajectory. Furthermore, assume one can execute infinitely many runs of scenarios. 
\end{assumption}
Although the above assumption is not strictly valid in real-world tests, it is applicable in the computer simulations, which is the primary focus of this paper.
\begin{remark}
    In general, most statistical analyses requires that the feedback control policy $\pi(\cdot)$ generates actions that are independent and identically distributed (i.i.d.) w.r.t. the distribution $\mathbb{P}$ on the admissible action space $\Gamma \subseteq \mathcal{U}$. It is thus a natural choice considering $\uu \sim U(\Gamma)$. Some also consider $\pi(\cdot)$ as a more specific policy with a deterministic or stochastic mapping to $\mathcal{U}$ inspired by naturalistic behavioral studies and model insights as we will show in Section~\ref{sec:validation}.
\end{remark}

For both the safety validation problem in Section~\ref{sec:validation} and the safety quantification problem in Section~\ref{sec:quantification}, we first derive problem formulations and possible solutions that at least ``deserve" the testing effort. That is, as more runs of scenarios are sampled, one is guaranteed to have more accurate safety assurance results and higher confidence regarding such assessment. We then derive finite-sampling analysis to analytically quantify the relation between the number of samples and the assessment outcome. Finally, we seek to improve the proposed algorithms with better sampling efficiency and time complexity without jeopardizing any obtained theoretical properties.

\section{The Safety Validation Problem}\label{sec:validation}
By Definition~\ref{def:safety}, ensuring the safety of SV in the given domain becomes closely related to validating if the domain is robustly controlled forward invariant. However, it is practically difficult to perfectly ensure the set invariance property without any knowledge of the dynamical model. The data-driven scenario sampling approach is thus adopted and is expected to solve the following problem. 

\begin{problem}\label{prob:val}
    \textbf{(Scenario Sampling Safety Validation Problem)} Given the scenario dynamics~\eqref{eq:snapshot-dyn}, $\Phi \subseteq \Sigma$ and $\Gamma \subseteq \mathcal{U}$, a \emph{scenario sampling safety validation} problem seeks to design a scenario sampling algorithm
    \begin{equation}
        \mathcal{VAL}(N, \Phi, \Gamma), \mathcal{VAL}: \Z \times \Sigma \times \mathcal{U} \rightarrow \{\text{True}, \text{False}\}
    \end{equation}
    with $N \in \Z$ samples. With a sufficiently large $N$, $\mathcal{VAL}(N, \Phi, \Gamma)$ is expected to determine whether $\Phi$ is robustly controlled forward invariant w.r.t. $\Gamma$.
\end{problem}

In practice, the control of the scenario propagation is not always vaguely defined through $\Gamma$. Some existing methodologies give specific deterministic testing policies. This is commonly observed in ADAS operational safety tests by the European New Car Assessment Program (NCAP)~\cite{van2017euro} and other government sectors. For ADS safety testing, the scenario is also concretely defined, such as in~\cite{feng2020testing, bagschik2018ontology}. Formally speaking, this creates the following simplified safety validation problem. 
\begin{problem}\label{prob:val_dp}
    \textbf{(Deterministic Scenario Sampling Safety Validation Problem)} Consider Problem~\ref{prob:val} with modified settings in three ways:
    \begin{enumerate}
        \item Assume the underlying unknown SV policy is time-invariant and deterministic.
        \item Let the feedback control policy of the scenario dynamics~\eqref{eq:snapshot-dyn} be time-invariant and deterministic, i.e., $\uu = \pi(\boldsymbol\sig), \pi: \Sigma \rightarrow \mathcal{U}$, or equivalently, $\Gamma(\boldsymbol\sig) = \{\pi(\boldsymbol\sig)\}$.
        \item Assume zero uncertainties in~\eqref{eq:snapshot-dyn}, i.e., $\boldsymbol\omega(t)=0, \forall t$.
    \end{enumerate}
    Correspondingly, we have the modified system of~\eqref{eq:snapshot-dyn} as 
    \begin{equation}\label{eq:snapshot-dynmod}
        \boldsymbol\sig(t+1)=f(\boldsymbol\sig(t), \pi(\boldsymbol\sig(t)))=g(\boldsymbol\sig(t)).
    \end{equation}
    As a result, with a sufficiently large $N$, $\mathcal{VAL}(N, \Phi, \Gamma)$ is expected to determine whether $\Phi$ is forward invariant for~\eqref{eq:snapshot-dynmod}.
\end{problem}
The listed assumptions essentially transfer a controlled forward invariant set validation problem to an invariant set validation problem by the above formulation.

For both Problem~\ref{prob:val} and Problem~\ref{prob:val_dp}, the essential property of the corresponding solution is to guarantee the so-called ``completeness" and obtain an ``almost" set invariance guarantee. In general, a complete algorithm returns a solution (or outcome) if one exists. An almost invariant set is a set that is mostly invariant except for an arbitrarily small subset. Depending on how one defines such an ``arbitrarily small subset", one can have either a resolution complete algorithm or a probabilistic complete algorithm. Both are presented in detail to solve Problem~\ref{prob:val_dp} in the next two subsections. The two different methodologies for achieving completeness and almost set invariance guarantee are then merged to solve the more complex Problem~\ref{prob:val}. Finally, some techniques that provably accelerate the proposed complete safety validation algorithms are presented.

\subsection{Resolution complete safety validation for Problem~\ref{prob:val_dp}}
Intuitively, a \emph{resolution complete} safety validation algorithm divides the given set $\Phi$ to multiple grids and validates each grid individually. The completeness is asymptotically obtained as one makes smaller subdivisions. The division of the given set is formally defined in Definition~\ref{def:delta-coverage} and is also conceptually illustrated in Figure~\ref{fig:delta-coverage}.
\begin{definition}\label{def:delta-coverage}
    \textbf{($\delta$-Covering Set)} Given $\delta\in\R_{>0}$ and $\Phi \subseteq \Sigma$. Let $\mathcal{N}_{\delta}(\boldsymbol\sig)$ be the $\delta$-neighborhood of $\boldsymbol\sig$, i.e., for all $\boldsymbol\sig' \in \mathcal{N}_{\delta}(\boldsymbol\sig), \norm{\boldsymbol\sig-\boldsymbol\sig'} \leq \delta$. The $\delta$-covering set $\Phi_{\delta}$ of $\Phi$ then satisfies
    \begin{equation}
        \Phi_{\delta}=\underset{i\in\Z_n}{\bigcup} \mathcal{N}_{\delta}(\boldsymbol\sigma_i) \supseteq \Phi,  \Phi_{\sigma}=\{\boldsymbol\sig_i\}_{i\in\Z_n} \subseteq \Phi
    \end{equation}
\end{definition}

\begin{figure}
    \centering
    \includegraphics[width=0.4\textwidth]{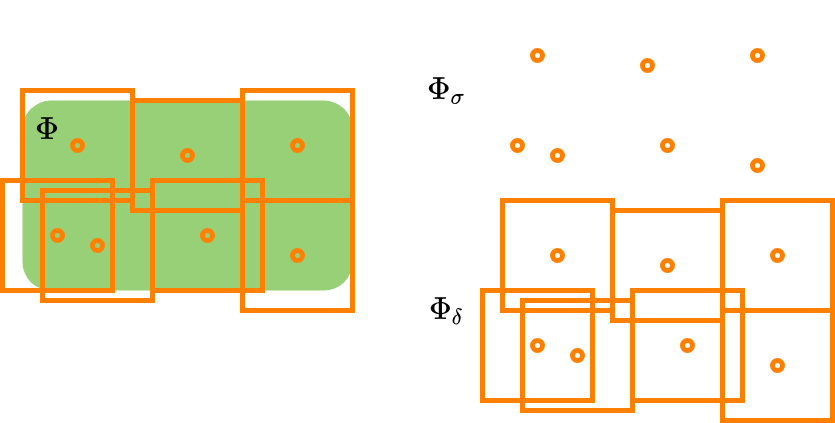}
    \caption{An example of the $\delta$-covering set on $\Phi \subseteq \R^2$.}
    \label{fig:delta-coverage}
    \vspace{-5mm}
\end{figure}

\begin{remark}
    Note that Definition~\ref{def:delta-coverage} is presented with a scalar value of $\delta \in \R_{>0}$ in conjunction with the $\ell_{\infty}$-norm selected for distance measure. The idea of the $\delta$-coverage can be further extended to work with the vector $\boldsymbol\delta\in\R_{>0}^n$ for high dimensional state space, but is out of the scope of this paper.
\end{remark}

We then have the formal definition of the \emph{$\delta$-almost robustly forward invariance} in Definition~\ref{def:delta-almost-inv}.
\begin{definition}\label{def:delta-almost-inv}
    \textbf{($\delta$-Almost Robustly Forward Invariance)} Given $\delta\in\R_{>0}$, $\Phi \subseteq \Sigma$, and all assumptions in Problem~\ref{prob:val_dp}. The set $\Phi$ is $\delta$-almost robustly forward invariant for the system~\eqref{eq:snapshot-dynmod} if there exists a $\delta$-covering set $\Phi_{\delta}$ of $\Phi$ with $\Phi_{\sigma}$, such that 
    \begin{equation}\label{eq:delta-almost-in}
        \mathbb{P}\Big(\big\{\boldsymbol\sig \in \Phi_{\sigma} : g(\boldsymbol\sigma,\uu) \in \Phi_{\delta}\big\}\Big) = 1.
    \end{equation}
\end{definition}

\begin{remark}\label{rmk:delta-helps}
    One practical advantage of introducing the $\delta$-covering set is that the ``point-in-set" check can be approximated as a point-to-set distance check. For example, identifying whether $g(\boldsymbol\sigma,\uu) \in \Phi_{\delta}$ is essentially checking whether $d(g(\boldsymbol\sigma,\uu), \Phi_{\delta})\leq \delta$. Such a simplification does not rely on the analytical geometric structure of $\Phi$, which can be very complex in practice. It helps solve the validation and, more importantly, the quantification problems effectively, as we will introduce later.
\end{remark}

\begin{algorithm}
    \begin{algorithmic}[1]
    \State {\bf Input:} $N, \Phi_{\delta}, \Phi_{\sig}$, the trajectory horizon $K$.
    \State {\bf Initialize: } $i=1$, done=False.
    \State {{\bf While} $i < N$ and not done:}
    \State {\ \ \ \ $\boldsymbol\sigma_0 = \Phi_{\sig}[i]$, $\mathcal{T} = \mathcal{RS}(\boldsymbol\sigma_0, K, \pi(\cdot))$}
    \State {\ \ \ \ $\boldsymbol\sigma=\mathcal{T}[2], j=2$}
    \State {\ \ \ \ {\bf While} $j < K$:}
    \State {\ \ \ \ \ \ \ \ {\bf If} $\mathcal{T}[j] \notin \Phi_{\delta}$}
    \State {\ \ \ \ \ \ \ \ \ \ \ \ done=True}
    \State {\ \ \ \ \ \ \ \ \ \ \ \ \textbf{break}}
    \State {\ \ \ \ \ \ \ \ {\bf End If}}
    \State {\ \ \ \ \ \ \ \ $j$+=$1$}
    \State {\ \ \ \ $i$+=$1$}
    \State {{\bf Output:} $\text{True}$ if $i=N$ else $\text{False}$}
    \end{algorithmic}
    \caption{$\delta$-Almost Safety Validation $\mathcal{VAL}^{\delta}(N, \Phi_{\delta}, \Phi_{\sig})$} \label{alg:val_delta}
\end{algorithm}

The $\delta$-covering set that satisfies the property by Definition~\ref{def:delta-almost-inv} is not necessarily unique. Algorithm~\ref{alg:val_delta} details the safety validation solution for a given $\delta$. Obviously, as $\delta$ tends to zero, one is expected to validate the invariance property of $\Phi$. The analytical correlation among the validation outcome by Algorithm~\ref{alg:val_delta}, the choice of $\delta$, and the number scenarios required are specified by Theorem~\ref{thm:res-complete-val} and Lemma~\ref{lma:res-complete-val}. 
\begin{theorem}\label{thm:res-complete-val}
    Let $E_{\Phi}$ be the event of $\Phi$ being not forward invariant for~\eqref{eq:snapshot-dynmod}. Algorithm~\ref{alg:val_delta} is resolution complete, i.e.,
    \begin{equation}
        \limsup_{\delta \rightarrow 0}\mathbb{P} \Bigg( \bigg\{ \mathbb{P}\big(E_{\Phi} \big) \leq \frac{(2\delta)^n}{|\Phi|} \bigg\} \Bigg) = 1.
    \end{equation}
\end{theorem}
Note that $|\Phi|$ is the cardinality of $\Phi$ and $n$ denotes the state space dimension defined in Section~\ref{sec:prelimilaries}.
\begin{lemma}\label{lma:res-complete-val}
    $\Phi$ is $\delta$-almost robustly forward invariant with probability one for~\eqref{eq:snapshot-dynmod} if
    \begin{equation}
        \mathcal{VAL}^{\delta}(N, \Phi_{\delta}, \Phi_{\sig})=\text{True} \text{ with }N \geq \frac{|\Phi|}{(2\delta)^n}.
    \end{equation}
\end{lemma}
Proof of Theorem~\ref{thm:res-complete-val} is straightforward from observation, i.e., the minimum number of $\delta$-neighborhoods required for validation is calculated by the ratio between the cardinality of the sample space $|\Phi|$ and the estimated cardinality of the $\delta$-neighborhood $(2\delta)^n$. Lemma~\ref{lma:res-complete-val} is directly obtained from Theorem~\ref{thm:res-complete-val}. As the state space dimension $n$ grows large, e.g. the scenario system involving more vehicles and controllable factors, the required number of samples $N$ dictated by Lemma~\ref{lma:res-complete-val} also significantly increases. Technically, Algorithm~\ref{alg:val_delta} is not a scenario sampling algorithm, given the completeness guarantee in Theorem~\ref{thm:res-complete-val} requires a explicit number of samples and the confidence level is one as long as each $\delta$-neighborhood is visited. Yet the analysis still holds its value for various proposed algorithms later.

\subsection{Probabilistic complete safety validation for Problem~\ref{prob:val_dp}}
The probabilistic complete solution for Problem~\ref{prob:val_dp} takes a different approach from the resolution complete methodology discussed above. Definition~\ref{def:delta-almost-inv} specifies the arbitrarily small subset through $\delta$, which comes with a strict topological interpretation through the $\delta$-covering set. On the other hand, the $\epsilon$-almost robustly controlled forward invariance definition specifies the small subset in a Monte Carlo manner. The primary technique discussed in this subsection adapts from~\cite{wang2020scenario}.

\begin{definition}\label{def:epsilon-almost-inv}
    \textbf{($\epsilon$-Almost Robustly Forward Invariance)} Given $\epsilon \in (0,1]$, $\Phi \subseteq \Sigma$, and all assumptions in Problem~\ref{prob:val_dp}. The set $\Phi$ is $\epsilon$-almost robustly forward invariant for the system~\eqref{eq:snapshot-dynmod} if 
    \begin{equation}\label{eq:epsilon-almost-inv}
        \mathbb{P}\Big(\big\{\boldsymbol\sig \in \Phi : g(\boldsymbol\sigma) \not\in \Phi\big\}\Big)\leq \epsilon.
    \end{equation}
\end{definition}
We further consider Algorithm~\ref{alg:val_epsilon} for the $\epsilon$-almost set invariance validation. The intuitive idea is that if a significant amount of initial configurations leading to runs of scenarios that remain inside the given set, the set is then ``probabilistically" forward invariant. In particular, the algorithm executes the Monte Carlo sampling in the sense that the ratio between the total number of points in $\Phi^N_{j}$ and the number of all initial points leads to an estimate of the measure of the probability for the set $\Phi^N_0$ being forward invariant up to the $j$-th step. With sufficiently large $N$, $\Phi^N_0$ is arbitrarily close to $\Phi$ with arbitrarily high probability. This correlation is formally quantified through Theorem~\ref{thm:prob-complete-val} and Lemma~\ref{lma:prob-complete-val}.

\begin{algorithm}[H]
    \begin{algorithmic}[1]
    \State {\bf Input:} Sample size $N$, trajectory horizon $K$, sample space $\Phi$, parameter $\epsilon$.
    \State {\bf Initialize: } $j=1$, done=False, $N$ initial states $\Phi^N_0=\{\boldsymbol\sig_i\}_{i\in\Z_N} \subseteq \Phi$.
    \State {{\bf While} $j < K$ and not done:}
    \State {\ \ \ \ $i=1, \Phi^N_i=\emptyset$}
    \State {\ \ \ \ {\bf While} $i < N$:}
    \State {\ \ \ \ \ \ \ \ $\mathcal{T} = \mathcal{RS}(\Phi^N_{j-1}[i], 2, \pi(\cdot))$}
    \State {\ \ \ \ \ \ \ \ $\boldsymbol\sigma'=\mathcal{T}[2]$}
    \State {\ \ \ \ \ \ \ \ $\Phi^N_j.\text{append}(\boldsymbol\sigma')$}
    \State {\ \ \ \ \ \ \ \ {\bf If} $\boldsymbol\sigma' \notin \Phi$}
    \State {\ \ \ \ \ \ \ \ \ \ \ \ done=True}
    \State {\ \ \ \ \ \ \ \ \ \ \ \ \textbf{break}}
    \State {\ \ \ \ \ \ \ \ {\bf End If}}
    \State {\ \ \ \ \ \ \ \ $j$+=$1$}
    \State {\ \ \ \ $i$+=$1$}
    \State {{\bf Output:} $\text{True}$ if $i=N$ else $\text{False}$}
    \end{algorithmic}
    \caption{$\epsilon$-Almost Safety Validation $\mathcal{VAL}^{\epsilon}(N, \Phi)$} \label{alg:val_epsilon}
\end{algorithm}

\begin{theorem}\label{thm:prob-complete-val}
    Let $E_{\Phi}$ be the event of $\Phi$ being not forward invariant for~\eqref{eq:snapshot-dynmod}. Let the set of $N$ initial states $\Phi^N_0$ be i.i.d. w.r.t. the underlying distribution on $\Phi$. Algorithm~\ref{alg:val_epsilon} is probabilistic complete, i.e.,
    \begin{equation}\label{eq:prob-complete-p}
        \limsup_{\epsilon \rightarrow 0}\mathbb{P} \Bigg( \bigg\{ \mathbb{P}\big(E_{\Phi} \big) \geq \epsilon  \bigg\} \Bigg) \leq (1-\epsilon)^N.
    \end{equation}
\end{theorem}
\begin{lemma}\label{lma:prob-complete-val}
    $\Phi$ is $\epsilon$-almost robustly controlled forward invariant with probability no smaller than $1-\beta$ if
    \begin{equation}\label{eq:prob-complete-N}
        \mathcal{VAL}^{\epsilon}(N, \Phi)=\text{True} \text{ with }N \geq \frac{\ln{\beta}}{\ln{(1-\epsilon)}}.
    \end{equation}
\end{lemma}
The main proof of Theorem~\ref{thm:prob-complete-val} adapts from the proof of Theorem 2 in~\cite{wang2020scenario}. The bound $(1-\epsilon)^N$ is also obvious if the underlying set $\Phi$ is indeed $\epsilon$-almost forward invariant by Definition~\ref{def:epsilon-almost-inv}. Lemma~\ref{lma:prob-complete-val} is a direct outcome of Theorem~\ref{thm:prob-complete-val}. As a result, one requires at least $\frac{\ln{\beta}}{\ln{(1-\epsilon)}}$ initial samples in $\Phi$ to get a priori guarantee of the set $\Phi$ being forward invariant for~\eqref{eq:snapshot-dynmod} with probability at least $1-\beta$ for the given $\epsilon$ and $\beta$. For the desired confidence level of $0.99$ and $\epsilon=0.001$, Lemma~\ref{lma:prob-complete-val} indicates 4603 sampled runs of scenarios. 

\begin{remark}
    In theory, the trajectory length $K$ is irrelevant to the algorithm completeness and other statistical properties. In practice, the gathered states starting from the second step are not necessarily i.i.d. This does not affect the presented theoretical properties, as long as $K \geq 2$ and the number of initial sampled points are sufficient and i.i.d. over the initial sample space. The necessity of running scenarios with large $K$ is intuitively appealing and practically efficient. It is of future interest to investigate possible theoretical benefits this may bring to help solve the scenario sampling safety assurance problems. 
\end{remark}

Although the state space properties, such as the cardinality $|\Phi|$, are not posed as direct hyper-parameters in~\eqref{eq:prob-complete-p}, it is still implicitly affecting the sampling complexity. Given the same $\epsilon$, the system with larger $|\Phi|$ also comes with a larger subset that is not forward invariant.

\subsection{Probabilistic complete safety validation for Problem~\ref{prob:val}}

As we move on to solve Problem~\ref{prob:val}, the resolution complete solution becomes infeasible. Simply visiting every state-action pair in $\Phi \times \Gamma$ once is no longer sufficient to guarantee set invariance or safety. The trajectory propagation is no longer deterministic as assumed in solving Problem~\ref{prob:val_dp}. Hence one requires running scenarios repeatedly from the same initial state. On the other hand, the probabilistic complete solution remains valid, but brings practical challenges as we move on to the safety quantification problem later. This section presents an approach that combines the resolution complete solution with the probabilistic complete algorithm. It will also serve as a crucial piece in the safety quantification algorithms presented in the next section. The basic idea is revealed through the following definition.

\begin{definition}\label{def:almost-cpis}
    \textbf{($\epsilon\delta$-Almost Robustly Controlled Forward Invariance)} Given $\epsilon \in (0,1]$, $\delta\in\R_{>0}$, $\Phi \subseteq \Sigma$, and $\Gamma \subseteq \mathcal{U}$. The set $\Phi$ is $\epsilon\delta$-almost robustly controlled forward invariant w.r.t. $\Gamma$ for the system~\eqref{eq:snapshot-dyn} if there exists a $\delta$-covering set $\Phi_{\delta}$ of $\Phi$ with $\Phi_{\sigma}$ such that
    \begin{equation}\label{eq:almost-cpis}
        \mathbb{P}\Big(\big\{\uu \in \Gamma, \boldsymbol\sig \in \Phi_{\sigma} : f(\boldsymbol\sigma,\uu; \boldsymbol\omega) \not\in \Phi_{\delta}\big\}\Big)\leq \epsilon.
    \end{equation}
\end{definition}

The $\epsilon\delta$-almost set invariance is very similar to the $\epsilon$-almost set invariance given by Definition~\ref{def:epsilon-almost-inv}. The main difference is that the state sampling space has been simplified from $\Phi$ to the $\delta$-coverage of $\Phi$. As addressed by Remark~\ref{rmk:delta-helps}, this brings significant practical value for the algorithms introduced later. As a result, Algorithm~\ref{alg:val_delta_epsilon} is also similar to Algorithm~\ref{alg:val_epsilon} with the input sample space changed from $\Phi$ to the $\delta$-covering set of $\Phi$.

% The following algorithm is similar to Algorithm~\ref{alg:val_epsilon}.

Following Theorem~\ref{thm:prob-complete-val}, the probabilistic completeness still applies to Algorithm~\ref{alg:val_delta_epsilon} for actions being i.i.d. w.r.t. the underlying distribution on $\Gamma$. The lower bound for $N$ still applies, but the sample space is different from the state space only considered by Algorithm~\ref{alg:val_epsilon}.

Despite the completeness of the presented algorithm, sampling points (or trajectories) in such a brutal force manner is not necessarily a good idea in practice. The following section describes some provably valid improvements for the presented safety validation algorithms.

% \subsection{Probabilistic complete safety validation for Problem~\ref{prob:val}}
\begin{algorithm}[t]
    \begin{algorithmic}[1]
    \State {\bf Input:} Sample size $N$, trajectory horizon $K$, $\Phi_{\delta}, \Phi_{\sig}$, $\Gamma$, parameter $\epsilon$, $\delta$.
    \State {\bf Initialize: } $j=1$, done=False, $N$ initial states $\Phi^N_0=\{\boldsymbol\sig_i\}_{i\in\Z_N} \subseteq \Phi_{\sig}$.
    \State {{\bf While} $j < K$ and not done:}
    \State {\ \ \ \ $i=1, \Phi^N_j=\emptyset$}
    \State {\ \ \ \ {\bf While} $i < N$:}
    \State {\ \ \ \ \ \ \ \ $\mathcal{T} = \mathcal{RS}(\Phi^N_{j-1}[i], 2, U(\Gamma))$}
    \State {\ \ \ \ \ \ \ \ $\boldsymbol\sigma'=\mathcal{T}[2]$}
    \State {\ \ \ \ \ \ \ \ $\Phi^N_j.\text{append}(\boldsymbol\sigma')$}
    \State {\ \ \ \ \ \ \ \ {\bf If} $\boldsymbol\sigma' \notin \Phi_{\delta}$}
    \State {\ \ \ \ \ \ \ \ \ \ \ \ done=True}
    \State {\ \ \ \ \ \ \ \ \ \ \ \ \textbf{break}}
    \State {\ \ \ \ \ \ \ \ {\bf End If}}
    \State {\ \ \ \ \ \ \ \ $j$+=$1$}
    \State {\ \ \ \ $i$+=$1$}
    \State {{\bf Output:} $\text{True}$ if $i=N$ else $\text{False}$}
    \end{algorithmic}
    \caption{$\epsilon\delta$-Almost Safety Validation $\mathcal{VAL}^{\epsilon\delta}(N,\Phi_{\delta}, \Phi_{\sig}, \Gamma)$} \label{alg:val_delta_epsilon}
\end{algorithm}

% The following algorithm is similar to Algorithm~\ref{alg:val_epsilon}.

\begin{figure}[!t]
\centering
\begin{subfigure}{\linewidth}
  \centering
  \includegraphics[width=0.85\textwidth]{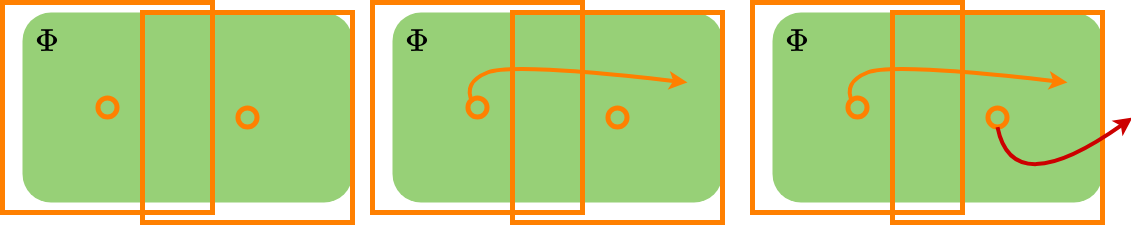}
  \caption{A conceptual illustration of Algorithm~\ref{alg:val_delta} with resolution completeness: at least one trajectory leaving $\Phi$ disqualifies the invariance property of $\Phi$.}
  \label{fig:val_delta}
\end{subfigure}
\begin{subfigure}{\linewidth}
  \centering
  \includegraphics[width=0.85\textwidth]{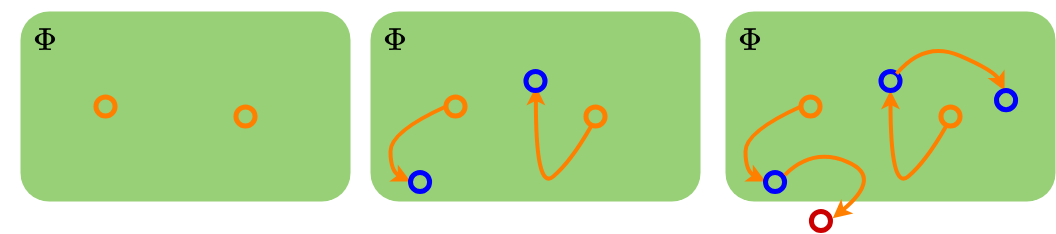}
  \caption{A conceptual illustration of Algorithm~\ref{alg:val_epsilon} with probabilistic completeness: a sufficient amount of sampled trajectories remain inside $\Phi$ for at least $K , K\geq2$ steps validates the forward invariance property.}
  \label{fig:val_epsilon}
\end{subfigure}
\begin{subfigure}{\linewidth}
  \centering
  \includegraphics[width=0.85\textwidth]{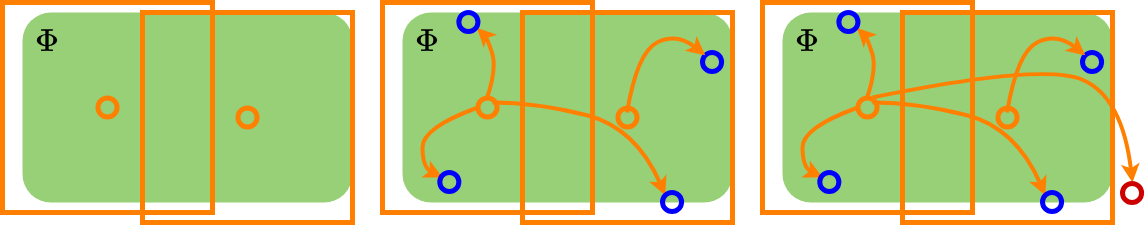}
  \caption{A conceptual illustration of Algorithm~\ref{alg:val_delta_epsilon}: note that the sample size is larger than the number of $\delta$-neighbourhoods given one of the point is repeatedly sampled for trajectory propagation.}
  \label{fig:val_epsilon_delta}
\end{subfigure}
\caption{Conceptual illustrations of three algorithms presented for the safety validation problem.}
\label{fig:val}
\end{figure}

\subsection{Improved safety validation}
The core idea for the presented improvements is that not every point in $\Phi \times \Gamma$ must be tested to prove forward invariance. This is illustrated in the state space $\Phi$ and the action space $\Gamma$ respectively as follows.

\subsubsection{The ``boundary" state set}
The first improvement on sampling efficiency is to show that the investigation of a subset of $\Phi$ is sufficient to solve Problem~\ref{prob:val}. Consider the following assumption.
\begin{assumption}\label{asp:bound_sig}
    Consider the dynamics~\eqref{eq:snapshot-dyn}, assume $\exists \bar{\sigma} > 0$ such that $\norm{\boldsymbol\sig(t+1)-\boldsymbol\sig(t)} \leq \bar{\sig}, \forall t \in \Z$.
\end{assumption}
Note that the above property can be a direct outcome of $f$ in~\eqref{eq:snapshot-dyn} being Lipschitz continuous in theory. It is also practically valid in the applications of interest to this paper.

\begin{theorem}\label{thm:boundary_val}
    Given $\epsilon \in (0, 1], \delta \in R_{>0}$ and Assumption~\ref{asp:bound_sig}. Let the compact sets $\Phi \subseteq \Sigma$ and $\Gamma \subseteq \mathcal{U}$. 
    Let $\partial \Phi$ be the boundary of $\Phi$. Define the $\bar{\sig}$-boundary set of $\Phi$ as
    \begin{equation}
        \Phi^{\bar{\sig}} = \{ \boldsymbol\sig \in \Phi \mid 0 \geq d_s(\boldsymbol\sig, \partial \Phi) \geq -\bar{\sig} \}.
    \end{equation}
    $\Phi$ is $\epsilon\delta$-almost robustly controlled forward invariant w.r.t. $\Gamma$ if and only if there exists a $\delta$-covering set of $\Phi^{\bar{\sig}}$ denoted as $\Phi^{\bar{\sig}}_{\delta}$ along with $\Phi^{\bar{\sig}}_{\sig}$ defined by Definition~\ref{def:delta-coverage}, such that
    \begin{equation}
        \mathbb{P}\Big(\big\{\uu \in \Gamma, \boldsymbol\sig \in \Phi^{\bar{\sig}}_{\sig} : f(\boldsymbol\sigma,\uu; \boldsymbol\omega) \not\in \Phi^{\bar{\sig}}_{\delta}\big\}\Big)\leq \epsilon.
    \end{equation}
\end{theorem}
\begin{proof}
    The sufficient condition is proved through Definition~\ref{def:almost-cpis}. If $\Phi$ is robustly controlled forward invariant for~\eqref{eq:snapshot-dyn}, any trajectory initialized from any subset of $\Phi$ shall still satisfy the invariance condition. The $\bar{\sig}$-boundary set of $\Phi$ is a subset of $\Phi$. Hence the sufficient condition is valid. On the other hand, the necessary condition is immediate through contradiction. 
\end{proof}
Theorem~\ref{thm:boundary_val} has reduced the state sample space to a subset near the boundary of $\Phi$ for sufficient validation. This also aligns with the intuition and engineering practice that one does not consider traffic participants that are sufficiently far away from the SV when executing the run of a testing scenario. Note that the boundary of $\Phi$ is not necessarily equivalent to collisions. Hence the boundary case is not the same as the so-called ``corner" cases in the concrete scenario design.

\subsubsection{The ``adversarial" action set}
Following a similar idea in pursuing a small subset for safety validation, we now consider the admissible action space $\Gamma$. Define the ``adversarial action set" as follows.
\begin{definition} \label{def:adv-aset}
    \textbf{(Adversarial action set)} Given $\boldsymbol\sig \in \Phi$ and the admissible action space $\Gamma$. The adversarial action set $\Gamma^* \subseteq \Gamma$ is defined as 
    \begin{equation}
        \Gamma^*:=\Big\{\uu\!\in\!\Gamma\!\mid\!\uu=\underset{\uu\in\Gamma}{\mathrm{ argmax }}\ \underset{\boldsymbol\omega\!\in\!\Omega}{\inf} \ d_s(f(\boldsymbol\sig,\!\uu;\!\boldsymbol\omega),\!\partial \Phi)\Big\}.
    \end{equation}
\end{definition}
That is, an action is deemed adversarial if it ``pushes" the current state as close as possible to the exterior of $\Phi$. Obviously, such an adversarial action is not necessarily unique in general. However, for some case-specific analyses, as we will show in the next section, the choices of adversarial actions can be very limited.

\begin{theorem}\label{thm:adv-aset}
    Consider the compact sets $\Phi \subseteq \Sigma$ and $\Gamma \subseteq \mathcal{U}$. Let $\Gamma^* \subseteq \Gamma$ be the adversarial action set by Definition~\ref{def:adv-aset}. $\Phi$ is robustly controlled forward invariant for~\eqref{eq:snapshot-dyn} w.r.t. $\Gamma$ if and only if $\Phi$ is robustly controlled forward invariant for~\eqref{eq:snapshot-dyn} w.r.t. $\Gamma^*$.
\end{theorem}
\begin{proof}
    First, if $\Phi$ is robustly controlled forward invariant for~\eqref{eq:snapshot-dyn} w.r.t. $\Gamma$, by Definition~\ref{def:robust-cfi}, it is immediate that $\Phi$ is robustly controlled forward invariant for~\eqref{eq:snapshot-dyn} w.r.t. any subset of $\Gamma$. Given $\Gamma^* \subseteq \Gamma$, this proves the sufficient condition. On the other hand, suppose $\Phi$ is robustly controlled forward invariant for~\eqref{eq:snapshot-dyn} w.r.t. $\Gamma^*$, but $\Phi$ is not robustly controlled forward invariant for~\eqref{eq:snapshot-dyn} w.r.t. $\Gamma$. This implies the existence of $\uu\in\Gamma \setminus \Gamma^*$ and $\uu^* \in \Gamma^*$ such that $d_s(f(\boldsymbol\sig, \uu; \boldsymbol\omega), \partial \Phi) \geq d_s(f(\boldsymbol\sig, \uu^*; \boldsymbol\omega)$ for some $\boldsymbol\sig \in \Phi$. This contradicts Definition~\ref{def:adv-aset}. Hence the necessary condition also holds.
\end{proof}
% The sufficient condition is a direct outcome of Definition~\ref{def:robust-cfi}, and the necessary condition can be shown through contradiction. 
To summarize, instead of taking $\Phi \times \Gamma$ for scenario sampling safety validation, it is provably sufficient to consider sampling within $\Phi^{\bar{\sig}}_{\sig} \times \Gamma^*$. This improves the sampling efficiency of all proposed algorithms in this section. 

\section{The Safety Quantification Problem} \label{sec:quantification}
Ultimately, the safety validation algorithm answers a ``yes-or-no" question. The presented results in the above section further enhance the answer with analytical guarantees. However, simply falsifying or certifying a given ODD is not sufficient for the general safety assurance task. This inspires the safety quantification problems studied in this section.

\subsection{The quantification problem formulation}
Overall, a scenario sampling safety quantification algorithm seeks to find a subset of the given sample space $\Sigma$ that is deemed safe. In practice, such a safe sub-domain, or an almost robustly controlled forward invariant subset of $\Sigma$, is not necessarily unique. This leads to two types of safety quantification problems. 
The first type of the scenario sampling quantification problem is to design an algorithm that is at least worth the sampling efforts. That is, as one samples more runs of scenarios, the algorithm is expected to have a higher probability of finding at least one safe sub-domain. This is formally presented as the \emph{feasible safety quantification problem} as follows.
\begin{problem}\label{prob:feasible-qnt}
    \textbf{(Feasible Scenario Sampling Safety Quantification Problem)} Given $\Sigma$ and $\mathcal{U}$, let $O(\Phi, \Gamma)$ denote the event of $\Phi \subseteq \Sigma$ being robustly controlled forward invariant w.r.t. $\Gamma \subseteq \mathcal{U}$ for~\eqref{eq:snapshot-dyn}. The feasible scenario sampling safety refinement problem seeks to find an algorithm
    \begin{equation}
        \mathcal{QNT}(N, \Sigma, \mathcal{U}), \mathcal{QNT}: \Z \times \Sigma \times \mathcal{U} \rightarrow \Sigma \times \mathcal{U},
    \end{equation}
    that is probabilistic complete, i.e., 
    \begin{equation}
        \limsup_{N\rightarrow \infty}\mathbb{P} \Bigg( \bigg\{ O\Big(\mathcal{QNT}\big(N, \Sigma, \mathcal{U}\big)\Big) \bigg\}  \Bigg) = 1.
    \end{equation}
\end{problem}

The second type of safety quantification problem focuses on when a particular desired property of the algorithm output is specified. For example, one may seek to obtain (i) the union of all safe sub-domains or (ii) the sub-domain with the cardinality close to a certain desired value. Such a desired property is often presented as a cost function, and the algorithm shall seek to reach the optimal cost at least asymptotically. This leads to the following problem.
\begin{problem}\label{prob:optimal-qnt}
    \textbf{(Optimal Scenario Sampling Safety Quantification Problem)} Given $\Sigma$, $\mathcal{U}$, and a cost function $c: \Sigma \times \mathcal{U} \rightarrow \R$ with the optimal cost $c^*$. Assume the optimal cost $c^*$ is feasible, i.e., $\exists \Phi^* \subseteq \Sigma, \exists \Gamma^* \subseteq \mathcal{U}$, $c(\Phi^*, \Gamma^*)=c^*$. The optimal scenario sampling safety refinement problem seeks to find an algorithm
    \begin{equation}
        \mathcal{QNT}(N, \Sigma, \mathcal{U}), \mathcal{QNT}: \Z \times \Sigma \times \mathcal{U} \rightarrow \Sigma \times \mathcal{U},
    \end{equation}
    that is asymptotically optimal, i.e.,
    \begin{enumerate}
        \item $\mathcal{QNT}(N, \Sigma, \mathcal{U})$ is probabilistic complete, and
        \item $\mathbb{P}\left(\bigg\{\underset{N\rightarrow \infty}{\limsup} \Big(c\big(\mathcal{QNT}(N, \Sigma, \mathcal{U})\big)=c^*\Big) \bigg\}\right)=1$.
    \end{enumerate}
\end{problem}

The presented optimal quantification problem is obviously more challenging than Problem~\ref{prob:feasible-qnt} given the probabilistic completeness is only part of the desired algorithm property. In practice, not every ``seemingly" working algorithm provably satisfies the desired properties specified by Problem~\ref{prob:feasible-qnt} or Problem~\ref{prob:optimal-qnt}. This will be illustrated in detail later, where a variety of scenario sampling algorithms of different safety assurance properties are presented. 

\begin{remark}\label{rmk:unique}
    Note that neither the probabilistic complete solution for Problem~\ref{prob:feasible-qnt} nor the asymptotically optimal solution for Problem~\ref{prob:optimal-qnt} is unique. With the same statistical guarantee, algorithms can still have different sampling efficiency, time complexity, and parallel processing compatibility. For readers that are familiar with the sampling-based motion planning context~\cite{karaman2011sampling}, the same property is observed. For example, $k$-nearest sPRM for $k=1$ is incomplete, RRT is probabilistic complete but sub-optimal, and RRT* is asymptotically optimal.
\end{remark}

The remainder of this section seeks to provide some algorithm examples that either (i) provably solves some of the above problems or (ii) provably fails to solve the problem. We believe the theoretical exploration towards both directions helps to enhance the understanding of the scenario sampling safety quantification problem from the set invariance perspective. Moreover, as addressed by Remark~\ref{rmk:unique}, it is beyond the scope of this paper to present ``the best" algorithm for the problems. It is of future interest to explore more efficient solutions satisfying the desired properties.

\subsection{Towards feasible safety quantification}
We start with a somewhat intuitive solution. Given Algorithm~\ref{alg:val_delta_epsilon} for provably complete safety validation, Algorithm~\ref{alg:qnt_vs} executes a ``Vanilla Sampling" (VS) scheme by repeatedly checking a sampled subset of $\Phi\times\Gamma$ until one is validated safe.
\begin{algorithm}[H]
    \begin{algorithmic}[1]
    \State {\bf Input:} Sample size $N \in \Z$, sample space $\Sigma \times \mathcal{U}$
    \State {{\bf For} $i$ in $\{1,\ldots, N\}$ {\bf do}:}
    \State {\ \ \ \ Sample $\Phi \subseteq \Sigma, \Gamma \subseteq \mathcal{U}$}
    \State {\ \ \ \ {\bf If} $\mathcal{VAL}^{\epsilon\delta}(N, \Phi, \Gamma)$}
    \State {\ \ \ \ \ \ \ \ {\bf break}}
    \State {\ \ \ \ {\bf End If}}
    \State {{\bf End For}}
    \State {{\bf Output:} $\Phi, \Gamma$}
    \end{algorithmic}
    \caption{$\mathcal{QNT}^{\text{VS}}(N, \Sigma, \mathcal{U})$} \label{alg:qnt_vs}
\end{algorithm}
However, a rigorous analysis of Algorithm~\ref{alg:qnt_vs} indicates a not so demanded property.
\begin{theorem}\label{thm:vs-incomplete}
    Algorithm~\ref{alg:qnt_vs} is probabilistic incomplete.
\end{theorem}

The main idea to prove the Theorem~\ref{thm:vs-incomplete} is to show that with the consistent sample space $\Phi \times \Gamma$, the probability of getting an $\epsilon\delta$-almost robustly controlled forward invariant set is a constant $c, c \in [0,1]$. Hence it does not change with respect to the number of samples $N$. 

For the algorithm variants discussed in the remainder of this section, we also adopt Assumption~\ref{asp:fix-gamma}.
\begin{assumption}\label{asp:fix-gamma}
    Assume the admissible action space $\Gamma$ is given and remains time-invariant.
\end{assumption}
As a result, the quantification algorithm is not allowed to modify $\Gamma$ in searching for the invariant set. In practice, fixing $\Gamma$ indicates a fixed control capability of other traffic participants and scenario factors, which further induces a fixed level of traffic aggressiveness. This assumption simplifies the analysis, but the proposed algorithms are still compatible with the general case where $\Gamma$ remains unknown. 

We are now ready to propose another alternative in Algorithm~\ref{alg:qnt_te} with ``$\delta$-pruning" where the disqualified samples with the corresponding $\delta$-neighborhoods are instantaneously removed from the sample space in an ``trial-and-error" manner.
\begin{algorithm}[H]
    \begin{algorithmic}[1]
    \State {\bf Input:} Sample size $N$, sample space $\Sigma$, admissible action space $\Gamma$, parameters $\epsilon, \delta$, trajecoty horizon $K$.
    \State {\bf Initialize: } $i=1$, $\Phi_{\delta}, \Phi_{\sig}$ obtained from the $\delta$-coverage of $\Phi=\Sigma$, $\bar{\Phi}=\emptyset$
    \State {{\bf While} $i<N$:}
    \State {\ \ \ \ Sample $\boldsymbol\sig \in \Phi_{\sig} \setminus \bar{\Phi}$}
    \State {\ \ \ \ $\mathcal{T} = \mathcal{RS}(\boldsymbol\sig, K, U(\Gamma))$}
    \State {\ \ \ \ {\bf For} $\boldsymbol\sig'$ in $\mathcal{T}[2:]$ {\bf do}}
    \State {\ \ \ \ \ \ \ \ {\bf If} $\boldsymbol\sig' \notin \Phi_{\delta}$}
    \State {\ \ \ \ \ \ \ \ \ \ \ \ $\bar{\Phi}\text{.append}(\boldsymbol\sigma)$}
    \State {\ \ \ \ \ \ \ \ \ \ \ \ {\bf break}}
    \State {\ \ \ \ \ \ \ \ {\bf End If}}
    \State {\ \ \ \ {\bf End For}}
    \State {\ \ \ \ $i$+=$1$}
    \State {{\bf Output:} $\bigcup_{\boldsymbol\sig\in\Phi_{\sig} \setminus \bar{\Phi}}\ \mathcal{N}_{\delta}(\boldsymbol\sig)$, $\Gamma$}
    \end{algorithmic}
    \caption{$\delta$-Pruning Safety Quantification $\mathcal{QNT}^{\delta\text{P}}(N, \Sigma, \Gamma)$} \label{alg:qnt_te}
\end{algorithm}
Figure~\ref{fig:qnt_te} conceptually illustrates the working process of Algorithm~\ref{alg:qnt_te}. Intuitively, the idea of falsifying the close neighborhood of a disqualified trajectory is appealing and may still work in some cases in practice. However, a more detailed analysis reveals a potential flaw of the algorithm.

\begin{figure}
    \vspace{3mm}
    \centering
    \includegraphics[width=0.6\textwidth]{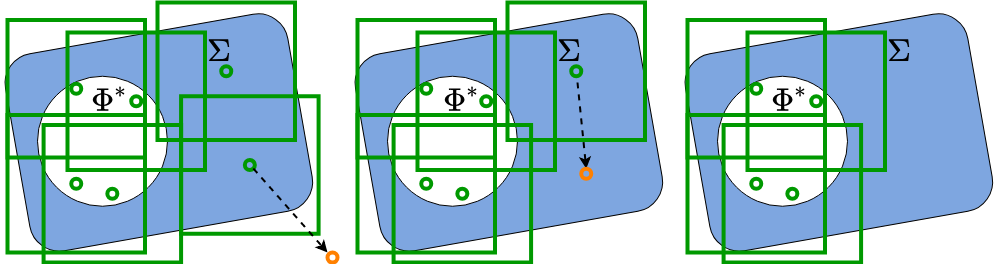}
    \caption{A conceptual illustration of Algorithm~\ref{alg:qnt_te}: the $\delta$-neighborhood $\mathcal{N}_{\delta}(\boldsymbol\sig)$ will be removed from the instantaneous sample space $\Phi_{\sig}$ whenever a trajectory starting from $\boldsymbol\sig$ leaves the instantaneous $\delta$-covering set $\Phi_{\delta}$.}
    \label{fig:qnt_te}
    \vspace{-3mm}
\end{figure}
\begin{figure}
    \centering
    \includegraphics[trim=0 1cm 0 0cm, clip, width=0.6\textwidth]{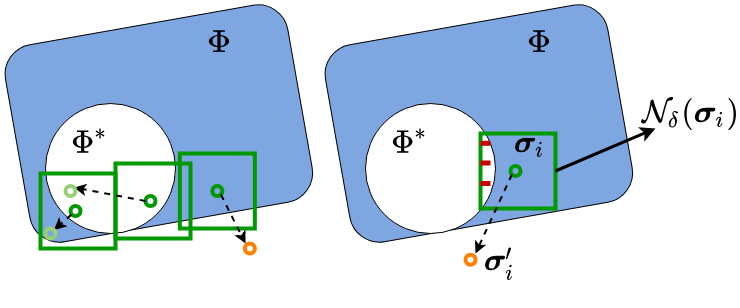}
    \caption{A counter example illustrating the incompleteness of Algorithm~\ref{alg:qnt_te}: as shown in the right subplot, suppose the one-step simulation of~\eqref{eq:snapshot-dyn} from $\boldsymbol\sig$ obtains $\boldsymbol\sig' \notin \Phi_{\delta}$, one shall remove $\mathcal{N}_{\delta}(\boldsymbol\sig)$ from the sample space, this may lead to the removal of the subset that is part of the underlying almost robustly controlled forward invariant set $\Phi^*$ (denoted in the red color). Meanwhile, the left subplot shows some cases when Algorithm~\ref{alg:qnt_te} operates appropriately.}
    \label{fig:qnt_te_incomplete}
    \vspace{-3mm}
\end{figure}

\begin{theorem}\label{thm:te-incomplete}
    Algorithm~\ref{alg:qnt_te} is probabilistic incomplete.
\end{theorem}
The proof of Theorem~\ref{thm:te-incomplete} can be shown by contradiction. A conceptual case is shown in Figure~\ref{fig:qnt_te_incomplete}. The fundamental problem is that the disqualification of a certain $\boldsymbol\sig$ can not be generalized to every point in its $\delta$-neighborhood. Hence $\mathcal{N}_{\delta}(\boldsymbol\sig)$ cannot be ruled out of consideration. 

To a certain extent, Algorithm~\ref{alg:qnt_te} performs the set search hoping that one would converge to a feasible solution by continually removing unqualified points from the set. Thus, it is a natural extension to consider the set search from the opposite direction presented through the Algorithm~\ref{alg:qnt_ae} and Figure~\ref{fig:qnt_ae} with ``adaptive exploration".

\begin{algorithm}
    \begin{algorithmic}[1]
    \State {\bf Input:} Sample size $N \in \Z$, state sample space $\Sigma$, admissible action space $\Gamma$, trajectory horizon $K$, parameters $\epsilon, \beta, \delta_0, \gamma\in(0,1)$.
    \State {\bf Initialize: } $N_{\epsilon}=\frac{\ln{\beta}}{\ln{(1-\epsilon)}}$, $\delta_t=\delta_0$, $\Phi_{\sig} = \{\boldsymbol\sigma_0\}, \boldsymbol\sig_0 \sim U(\Sigma), i=1, n=1$
    \State {{\bf While} $n<N$:}
    \State {\ \ \ \ {\bf While} $i<N_{\epsilon}$:}
    \State {\ \ \ \ \ \ \ \ $\boldsymbol\sigma \sim U(\Phi_{\sig})$}
    \State {\ \ \ \ \ \ \ \ $\mathcal{T}=\mathcal{RS}(\boldsymbol\sigma, K, U(\Gamma))$}
    \State {\ \ \ \ \ \ \ \ {\bf For} $\boldsymbol\sig'$ in $\mathcal{T}[2:]$ {\bf do}}
    \State {\ \ \ \ \ \ \ \ \ \ \ \ {\bf If} $\boldsymbol\sig' \notin \Sigma$}
    \State {\ \ \ \ \ \ \ \ \ \ \ \ \ \ \ \ $\Phi_{\sig} = \{\boldsymbol\sigma_0\}, \boldsymbol\sig_0 \sim U(\Sigma)$}
    \State {\ \ \ \ \ \ \ \ \ \ \ \ \ \ \ \ $i=1$}
    \State {\ \ \ \ \ \ \ \ \ \ \ \ \ \ \ \ {\bf break}}
    \State {\ \ \ \ \ \ \ \ \ \ \ \ {\bf End If}}
    \State {\ \ \ \ \ \ \ \ \ \ \ \ {\bf If} $\boldsymbol\sigma' \notin \Phi_{\delta_t}$}
    \State {\ \ \ \ \ \ \ \ \ \ \ \ \ \ \ \ $\Phi_{\sig}\text{.append}(\boldsymbol\sigma')$}
    \State {\ \ \ \ \ \ \ \ \ \ \ \ \ \ \ \ $i=1$}
    \State {\ \ \ \ \ \ \ \ \ \ \ \ {\bf End If}}
    \State {\ \ \ \ \ \ \ \ {\bf End For}}
    \State {\ \ \ \ \ \ \ \ {\bf If} $n=N$}
    \State {\ \ \ \ \ \ \ \ \ \ \ \ {\bf break}}
    \State {\ \ \ \ \ \ \ \ {\bf End If}}
    \State {\ \ \ \ \ \ \ \ $i$+=$1$, $n$+=$1$}
    \State {\ \ \ \ $\delta_t=\gamma\delta_t$, update $\Phi_{\delta_t}$ w.r.t. $\Phi_{\sig}$}
    \State {{\bf Output:} $\Phi_{\delta_t}, \Gamma$}
    \end{algorithmic}
    \caption{Adaptive Exploration Safety Validation $\mathcal{QNT}^{\text{AE}}(N, \Sigma, \Gamma)$} \label{alg:qnt_ae}
\end{algorithm}

\begin{figure*}
    \vspace{3mm}
    \centering
    \includegraphics[width=0.9\linewidth]{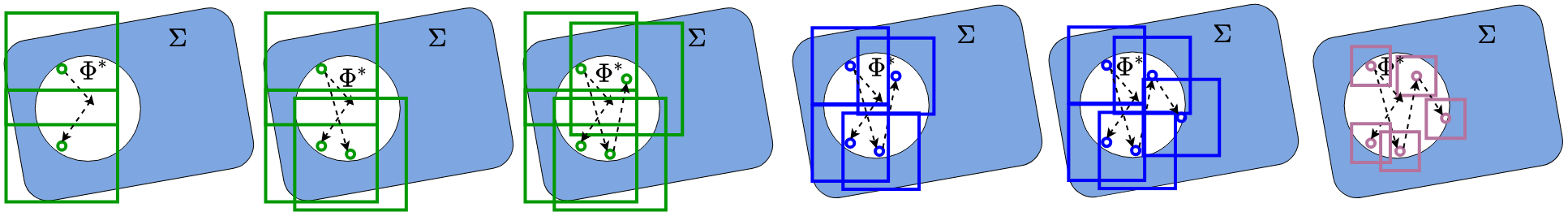}
    \caption{A conceptual illustration of Algorithm~\ref{alg:qnt_ae}.}
    \label{fig:qnt_ae}
    \vspace{-3mm}
\end{figure*}

Intuitively, Algorithm~\ref{alg:qnt_ae} recursively executes a two-step procedure. First, for the instantaneous $\delta_t$, one keeps expanding its reachable domain by recursively adding newly found states to the domain, which ultimately reaches an $\epsilon\delta_t$-almost robustly controlled forward invariant set. One then takes the second step of modifying $\delta_t$ by multiplying the decay coefficient $\gamma$, which potentially creates ``holes" for future exploration to fill.
With only mild assumption upon the initialization of $\boldsymbol\sigma_0$, we have the following result.
\begin{theorem}\label{thm:ie-complete}
    Assume there exists a robustly controlled forward invariant set $\Phi' \subset \Sigma$ and $\boldsymbol\sigma_0 \in \bar{\Phi}$. Let the actions $\uu$ be i.i.d. w.r.t. the underlying distribution on $\Gamma$. Algorithm~\ref{alg:qnt_ae} is probabilistic complete.
\end{theorem}
\begin{proof}
First, consider fixing the choice of $\delta_t=\delta_0$, i.e. $\gamma=1$. Algorithm~\ref{alg:qnt_ae} ensures the non-decreasing property of $|\Phi_{\sig}|$, i.e., $\Phi_{\sig}$ either expands with new added states or remains unchanged. Consider the assumed $\Phi'$ in the theorem is compact and $\boldsymbol\sigma_0 \in \Phi'$, there shall exist $M \in \Z$ and $M < \infty$ such that $\forall i \geq M, i\in Z$, $\Phi_{\sig}$ no longer expands. To this end, the execution process of Algorithm~\ref{alg:qnt_ae} is further divided into two consecutive stages. Before reaching the $M$-th iteration, one keeps searching for new added states to fill in $\Phi_{\sig}$. The second stage (after the $M$-th iteration) is essentially a safety validation problem. Applying Theorem~\ref{thm:prob-complete-val} one directly has the desired $\epsilon\delta$-almost robustly controlled forward invariance property with $N_{\epsilon}$ consecutive samples. Second, for $\gamma \in (0,1)$, one repeatedly executes the first step for each unique $\delta_t$. As $N$ tends to infinity, the probabilistic completeness is obtained.
\end{proof}
\begin{remark} \label{rmk:unify_val_and_qnt}
    The above proof reveals an important intrinsic property that a safety quantification algorithm always terminates with a safety validation algorithm. This unifies what we have presented in Section~\ref{sec:validation} and~\ref{sec:quantification}.
\end{remark}

Recall that the robustly controlled forward invariant subset of $\Sigma$ is not necessarily unique. Algorithm~\ref{alg:qnt_ae} is guaranteed to find one of them but is not capable of finding all of them. This leads to the optimal safety quantification problem next.

\subsection{Towards optimal safety quantification}
Consider the following cost function for Problem~\ref{prob:optimal-qnt}  
\begin{equation}\label{eq:cost}
    c(\Phi, \Gamma) = -|\Phi \times \Gamma|, c: \Sigma \times \mathcal{U} \rightarrow \R_{\leq0}.   
\end{equation}
That is, one seeks to find the maximal invariant set, i.e., the union of all safe sub-domains for the SV. We have the following assessment regarding Algorithm~\ref{alg:qnt_ae}.
\begin{theorem}\label{thm:ie-nonoptimal}
    Consider the cost function~\eqref{eq:cost} and assume the optimal cost $c^*$ is feasible. Algorithm~\ref{alg:qnt_ae} is not asymptotically optimal.
\end{theorem}
The above theorem is immediate. For example, assume the existence of two robustly controlled forward invariant sets $\Phi_1^*$ and $\Phi_2^*$. Let $\Phi_1^* \cap \Phi_2^* = \emptyset$. Suppose the initialization $\boldsymbol\sigma_0 \in \Phi_1^*$. It follows that Algorithm~\ref{alg:qnt_ae} only guarantees the convergence to $\Phi_1^*$ and is not able to quantify $\Phi_2^*$.

\begin{figure*}
    \vspace{3mm}
    \centering
    \includegraphics[width=0.95\linewidth]{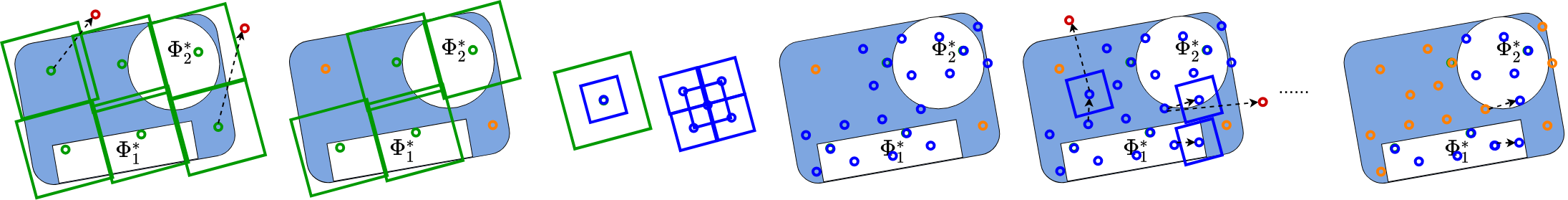}
    \caption{A conceptual illustration of Algorithm~\ref{alg:qnt_spe}.}
    \label{fig:qnt_spe}
    \vspace{-3mm}
\end{figure*}

The asymptotically optimal solution is not unique as addressed by Remark~\ref{rmk:unique}. For example, by executing Algorithm~\ref{alg:qnt_ae} in the ``multiple shooting" manner~\cite{bock1984multiple} where a sufficient number of $\boldsymbol\sigma_0$'s are initialized with sufficiently small $\delta_0$, the algorithm becomes asymptotically optimal for the cost function~\eqref{eq:cost}. However, the following proposal is practically more efficient than the multiple shooting extension of $\mathcal{QNT}^{\text{AE}}(N, \Sigma, \Gamma)$. In particular, we consider combing the advantages of Algorithm~\ref{alg:qnt_te} and Algorithm~\ref{alg:qnt_ae}, supplied with other structural improvements, leading to a ``Synchronous Pruning and Exploration" (SPE) scheme detailed in Algorithm~\ref{alg:qnt_spe} and Figure~\ref{fig:qnt_spe}. 

\begin{algorithm}[H]
    \begin{algorithmic}[1]
    \State {\bf Input:} Sample size $N \in \Z$ for validation, state sample space $\Sigma$, admissible action space $\Gamma$, trajectory horizon $K$, parameters $\epsilon, \delta_0, \gamma\in(0,1)$.
    \State {\bf Initialize: } $N_{\epsilon}=\frac{\ln{\beta}}{\ln{(1-\epsilon)}}$, $\delta_t=\delta_0$, $\Phi_{\delta_t}, \Phi_{\sig_t}$ obtained from the $\delta_t$-coverage of $\Phi=\Sigma$, $\bar{\Phi}=\emptyset$, set of edges $\Phi_{E_t}=\emptyset$, graph $G=(\Phi_{\sig_t}, \Phi_{E_t}), n=1$,
    \State {{\bf While} $n<N$:}
    \State {\ \ \ \ $i=1$}
    \State {\ \ \ \ {\bf While} $i<N_{\epsilon}$:}
    \State {\ \ \ \ \ \ \ \ $\boldsymbol\sigma \sim U(\Phi_{\sig_t})$}
    \State {\ \ \ \ \ \ \ \ $\mathcal{T}=\mathcal{RS}(\boldsymbol\sigma, K, U(\Gamma))$}
    \State {\ \ \ \ \ \ \ \ {\bf For} $\boldsymbol\sig'$ in $\mathcal{T}[2:]$ {\bf do}}
    \State {\ \ \ \ \ \ \ \ \ \ \ \ {\bf If} $\boldsymbol\sig' \notin \Sigma$}
    \State {\ \ \ \ \ \ \ \ \ \ \ \ \ \ \ \ $\Phi_{\sig_t}\text{.discard}(\text{reachable}(\boldsymbol\sig, G))$}
    \State {\ \ \ \ \ \ \ \ \ \ \ \ \ \ \ \ $\bar{\Phi}\text{.append}(\boldsymbol\sigma)$}
    \State {\ \ \ \ \ \ \ \ \ \ \ \ \ \ \ \ $i=1$}
    \State {\ \ \ \ \ \ \ \ \ \ \ \ \ \ \ \ {\bf break}}
    \State {\ \ \ \ \ \ \ \ \ \ \ \ {\bf End If}}
    \State {\ \ \ \ \ \ \ \ \ \ \ \ {\bf If} $d(\boldsymbol\sig', \Phi_{\delta_t})>\delta_t$ and $d(\boldsymbol\sig, \bar{\Phi})>\delta_t$}
    \State {\ \ \ \ \ \ \ \ \ \ \ \ \ \ \ \ $\Phi_{\sig_t}\text{.append}(\boldsymbol\sigma')$}
    \State {\ \ \ \ \ \ \ \ \ \ \ \ \ \ \ \ $\Phi_{E_t}\text{.append}((\boldsymbol\sig, \boldsymbol\sigma'))$}
    \State {\ \ \ \ \ \ \ \ \ \ \ \ \ \ \ \ $i=1$}
    \State {\ \ \ \ \ \ \ \ \ \ \ \ {\bf End If}}
    \State {\ \ \ \ \ \ \ \ {\bf End For}}
    \State {\ \ \ \ \ \ \ \ {\bf If} $n=N$}
    \State {\ \ \ \ \ \ \ \ \ \ \ \ {\bf break}}
    \State {\ \ \ \ \ \ \ \ {\bf End If}}
    \State {\ \ \ \ \ \ \ \ $i$+=$1$, $n$+=$1$}
    \State {\ \ \ \ $\Phi'_{\sig_t}$ obtained from the $\gamma\delta_t$-coverage of $\Phi_{\delta_t}$}
    \State {\ \ \ \ $\Phi_{\sig_t} = \Phi_{\sig_t} \cup \{\boldsymbol\sig\in\Phi'_{\sig_t} \mid d(\boldsymbol\sig, \bar{\Phi})>\gamma\delta_t\}$}
    \State {\ \ \ \ $\delta_t = \gamma\delta_t$}
    \State {{\bf Output:} $\Phi_{\delta_t}$, $\Gamma$}
    \end{algorithmic}
    \caption{$\mathcal{QNT}^{\text{SPE}}(N, \Sigma, \Gamma)$} \label{alg:qnt_spe}
\end{algorithm}

The algorithm starts by formulating a $\delta_t$-covering set of $\Sigma$. All $\boldsymbol\sig$'s are further formulated as vertices of a graph $G$. After that, the exploration consists of three main components, the pruning, the exploration, and the decay of $\delta_t$ when states are not being added to or removed from $\Phi_{\sig}$ for sufficiently many steps. Note that the function $\text{reachable}(\boldsymbol\sig, G)$ denotes all vertices that directly or indirectly connect to $\boldsymbol\sig$ on the graph $G$. We further emphasize the following points that make Algorithm~\ref{alg:qnt_spe} notably different from the previous proposed algorithms. First, comparing line 9 in Algorithm~\ref{alg:qnt_spe} with line 7 in Algorithm~\ref{alg:qnt_te} leads differnt pruning schemes. Second, the adaptive exploration in $\mathcal{QNT}^{\text{SPE}}(N, \Sigma, \Gamma)$ differs from that in $\mathcal{QNT}^{\text{AE}}(N, \Sigma, \Gamma)$ in two ways. (i) The newly added points cannot be too close to $\bar{\Phi}$ through the conditional check in line 15 of Algorithm~\ref{alg:qnt_spe}. (ii) Every time one decays $\delta_t$, not only that the existing $\delta_t$-neighborhoods get ``shrunk", the original $\Phi_{\delta_t}$ is also covered by the newly created covering set of $\Phi_{\gamma\delta_t}$ (line 25-27 in Algorithm~\ref{alg:qnt_spe}).

The optimal property of Algorithm~\ref{alg:qnt_spe} is further detailed in Theorem~\ref{thm:spe-optimal} in conjunction with Assumption~\ref{asp:sufficient_delta0}.

\begin{assumption}\label{asp:sufficient_delta0}
    Suppose the optimal solution $(\Phi^*, \Gamma)$ compatible with the cost function~\eqref{eq:cost} contains all of the $k$ robustly controlled forward invariant subsets of $\Sigma$, i.e., $\Phi^* = \bigcup_{i\in\Z_k}\ \Phi_i^*$. Assume the $\delta_0$ initialization of Algorithm~\ref{alg:qnt_spe} satisfies $(\delta_0/2)^n \leq \underset{i\in\Z_k}{\min}|\Phi_i^*|$.
    % \begin{equation}
    %     (\delta_0/2)^n \leq \underset{i\in\Z_k}{\min}|\Phi_i^*|
    % \end{equation}
\end{assumption}
Intuitively, by forcing $\delta_0$ to be sufficiently small, Assumption~\ref{asp:sufficient_delta0} ensures the existence of an initialized $\mathcal{N}_{\delta_0}(\boldsymbol\sigma)$ such that $\mathcal{N}_{\delta_0}(\boldsymbol\sigma)\subseteq\Phi_i^*, \forall i \in \Z_k$. In practice, {$\delta_0$} is given by user that varies with cases. It is always possible to make {$\delta_0$} sufficiently small to ensure Assumption{~\ref{asp:sufficient_delta0}}.
\begin{theorem}\label{thm:spe-optimal}
    Supplied with the cost function~\eqref{eq:cost}, Algorithm~\ref{alg:qnt_spe} ($\mathcal{QNT}^{\text{SPE}}(N, \Sigma, \Gamma)$) is asymptotically optimal for a sufficiently small $\delta_0$ deemed by Assumption~\ref{asp:sufficient_delta0}.
\end{theorem}

\begin{lemma}\label{lma:spe-optimal}
    Let $\delta_0$ be sufficiently small satisfying Assumption~\ref{asp:sufficient_delta0}. If the feedback control policy specified in line 7 with $U(\Gamma)$ is replaced with $U(\Gamma^*)$ with $\Gamma^*$ determined by the instantaneous $\Phi_{\delta_t}$ and Definition~\ref{def:adv-aset}, the algorithm remains asymptotically optimal.
\end{lemma}

The proof sketch of Theorem~\ref{thm:spe-optimal} follows a three-step argument. First, consider the adaptive exploration from a single initial point. The proof of Theorem~\ref{thm:ie-complete} can be adapted and the probabilistic completeness is guaranteed. Second, given Assumption~\ref{asp:sufficient_delta0}, it follows that for an arbitrary $\delta_t$-covering set of $\Sigma$, there exists at least one point initialized in each robustly controlled forward invariant subset of $\Sigma$. Hence the adaptive exploration sub-module of the Algorithm~\ref{alg:qnt_spe} is probabilistic complete and converges to at least a subset of $\Phi^*$. Finally, the pruning procedure is shown satisfying (i) no points in $\Phi_{\sig_t} \cap \Phi^*$ get removed, and (ii) all points in $\Phi_{\sig_t} \setminus (\Phi_{\sig_t} \cap \Phi^*)$ get removed, as number of samples $N$ tends to infinity. This completes the proof. Combining the above sketch with Definition~\ref{def:adv-aset} and Theorem~\ref{thm:adv-aset}, one can also obtain Lemma~\ref{lma:spe-optimal}.

In practice, Algorithm{~\ref{alg:qnt_spe}} does not execute indefinitely and the algorithm typically terminates if the sufficient accuracy is achieved to meet the choice of {$\epsilon$} and {$\delta$}. As a result, the total sample size {$N$} is a sufficiently large number determined by the user based on the studied case and other practical factors. When applied in practice, such as for the studied cases in the next section, the algorithm always converges to the desired accuracy (the set remains unchanged for {$N_{\epsilon}$} samples) before the sample limit is reached.

\begin{remark}\label{rmk:complexity}
    The computational complexity of Algorithm{~\ref{alg:qnt_spe}} is {$O(N^2)$} with the brute-force algorithm's complexity being {$O(N)$}. If one incorporates Theorem{~\ref{thm:boundary_val}} and Definition{~\ref{def:adv-aset}} by only investigating the adversarial action set and the boundary set of states, the complexity of Algorithm{~\ref{alg:qnt_spe}} can be further improved to be {$O(N\log{N})$}. However, as $N$ is related to the number of vertices on the graph, it essentially grows exponentially as the state-action dimension increases. In practice, for high-dimensional problems, one can relax the problem with appropriate choices of {$\delta$} and {$\epsilon$}.
\end{remark}

As discussed in Section{~\ref{sec:introduction}}, the computational efficiency between the proposed method and the backward reachability analysis is not directly comparable. In general, both methods suffer from the curse-of-dimensionality (CoD). For example, the complexity of the commonly adopted level-set method for backward reachable set approximation also grows exponentially along with the number of dimensions{~\cite{mitchell2004demonstrating}}.

It is worth emphasizing that the proposed Algorithm~\ref{alg:qnt_spe}, Theorem~\ref{thm:spe-optimal}, and Lemma~\ref{lma:spe-optimal} have unified most algorithms and key concepts introduced in the previous sections including the safety validation with Algorithm~\ref{alg:val_delta_epsilon}, the improved sampling scheme with the adversarial action set by Theorem~\ref{thm:adv-aset}, the pruning-based safety quantification with Algorithm~\ref{alg:qnt_te}, and the adaptive exploration considered by Algorithm~\ref{alg:qnt_ae}. Therefore, a successful numerical illustration of Algorithm~\ref{alg:qnt_spe} is also an empirical proof of the performance for many other introduced techniques. This is elaborated in the following case studies.

% \begin{proof}
%     We start from the simplified case where there is no $\delta$ decay, and then extends the proof to multiple steps with different $\delta$s.
%     Let the hyper-parameter remain as $\delta$ throughout the execution of the Algorithm~\ref{alg:qnt_spe}. 

%     The proof is two-fold. First is to show the incremental exploration converges to at least a sup-set of a $\delta$-covering set of $\Phi^*$. This is proved by contradiction. Suppose Algorithm~\ref{alg:qnt_spe} fails to find a sup-set of the $\delta$-covering set of the optimal $\Phi^*$. That is, there exists $\bar{\boldsymbol\sig} \in \Phi^*$,  $\mathcal{N}_{\delta}(\bar{\boldsymbol\sig})$ is not part of the algorithm's return. Consider $|\mathcal{N}_{2\delta}(\cdot)|\leq \min_{i\in\Z_k}|\Phi^*_i|$. This implies the existence of at least one $\delta$-neighbourhood initialized within $\Phi^*_i, \forall i \in \Z_k$. Apply Theorem~\ref{thm:ie-complete} the incremental exploration sub-module in Algorithm~\ref{alg:qnt_spe} is probabilistic complete. This contradicts the existence of $\bar{\boldsymbol\sig}$. We referred to the obtained sup-set as $\Phi'$. Second, consider the following set $\Phi'_{\sig}:= $

%     $\forall \boldsymbol\sigma' \in \Phi \setminus \Phi_{\delta}$, 
    
%     Second, we show the sup-set obtained

%     Hence Algorithm~\ref{alg:qnt_spe} finds at least one $\delta$-covering set of the optimal $\Phi^*$.
% \end{proof}

\subsection{Case study with the longitudinal driving system}
We study the class of longitudinal driving systems with emergency braking for collision avoidance. This covers most ADAS in practice, including AEB and Traffic Jam Assist (TJA). We present two examples, including (i) the safety quantification of the behavior agent in Carla~\cite{dosovitskiy2017carla} under various weather conditions, and (ii) a synthesized case study with a fleet of three vehicles leading to a more complex state space of five dimensions.

Before presenting the experiment details, we introduce two practical techniques that would enhance the performance of the proposed algorithm, prioritized sampling and experience replay. Both names are inspired by the Deep Q-Network (DQN) learning in the reinforcement learning field~\cite{mnih2013playing}. Similar methodologies are also observed in the scenario sampling context~\cite{feng2020testing, zhao2017accelerated}.

First, the prioritized sampling replaces the uniform sampling during the state exploration in Algorithm~\ref{alg:qnt_spe} (line 6) by assigning states that are closer to $\bar{\Phi}$ with a higher probability to be sampled. Formally speaking, for the $i$-th point $\boldsymbol\sig_i \in \Phi_{\sig_t}$, let $d_i=d(\boldsymbol\sig_i, \bar{\Phi})$, $k=|\Phi_{\sig_t}|$ and $d^*=\max_{i\in\Z_k}d_i$. Let $\alpha(\cdot)$ be a class $\mathcal{K}$ function. We then construct a probability vector $\boldsymbol\rho \in [0,1]^{k}$ satisfying $\boldsymbol\rho[i]=\frac{\alpha(d^*-d_i)}{\sum_{i\in\Z_k}\alpha(d^*-d_i)}$ as the sampling distribution on $\Phi_{\sig_t}$. Second, the experience replay saves all the visited trajectories as the ``experience" and replays the saved ones whenever the $\delta_t$ gets updated. This saves the practical efforts of running scenarios. Furthermore, the presented theoretical properties are not jeopardized by these two practical heuristics. It is of future interest to investigate whether prioritized sampling and experience replay would provably accelerate the quantification algorithm. 

Finally, it is worth emphasizing that techniques we introduced to enhance the safety validation solutions in Section~\ref{sec:validation} are also applicable and will be implemented on a case-specific basis for the following experiments.

\subsubsection{Lead-vehicle following system}
Consider the lead-vehicle following configuration shown in Figure~\ref{fig:lead_follow_demo}. The scenario state space $\Sigma$ consists of the SV velocity $v_0$, the lead vehicle velocity $v_1$, and the relative distance between the lead vehicle and the SV $p_{10}$. The action space $\Gamma \subseteq \mathcal{U}$ is the lead-vehicle's longitudinal control action space. The scenario propagation is controlled by the lead vehicle's longitudinal control policy subject to action saturation ($\mathcal{U}=[-5, 3](m/s^2)$ in Figure~\ref{fig:lead_follow_demo}).

\begin{figure}[!b]
    \centering
    \includegraphics[width=0.7\textwidth]{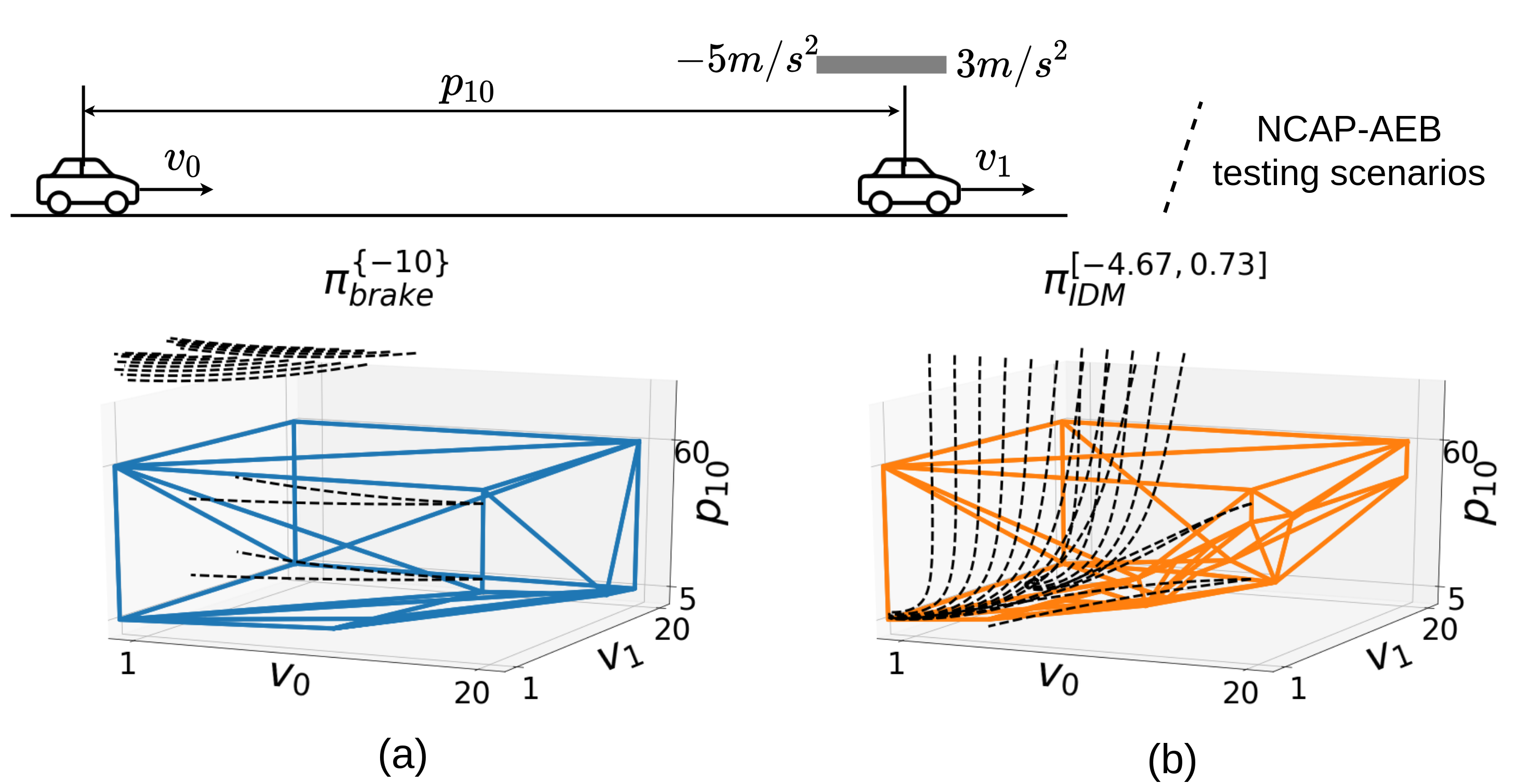}
    \captionsetup{singlelinecheck=off}
    \caption{A synthesized example of ``quantifying" the safety performance of two SV driving policies in the lead-vehicle following domain with Algorithm~\ref{alg:qnt_spe}: Two underlying SV policies are evaluated, (a) the SV keeps braking-to-stop at a constant rate of $-10 m/s^2$, (b) the SV executes the parameterized Intelligent Driving Model (IDM) with maximal braking capability set to $-4.67m/s^2$. For each SV policy, the $\epsilon\delta$-almost robustly controlled forward invariant set ($\epsilon=0.001, \delta=1$) is quantified with a confidence level of $0.9$ through the scenario sampling approach and illustrated by its convex hull approximation.}
    \label{fig:lead_follow_demo}
    \vspace{-3mm}
\end{figure}

Obviously, for sufficiently large $p_{10}$, the SV is safe regardless of $v_0, v_1$. Also, the safety analysis does not necessarily consider the entire admissible state space, e.g., extremely high velocity is not of interest. Therefore, during the set invariance analysis, some hard constraints are enforced. In the lead-vehicle following case, the upper-bound for $p_{10}$, the upper and the lower bounds for $v_0$ and $v_1$ are all hard constraints. The propagated states are truncated on default if the trajectory leaves the set through any of the facets induced by the hard constraints. This leaves only the bottom facet for analysis as shown in Figure~\ref{fig:lead_follow_demo} and Figure~\ref{fig:carla}. Furthermore, it is obvious that the adversarial action set with only the bottom facet of concern is the maximal braking maneuver, i.e., $\Gamma^*=\{-5\}$ for the synthesized examples in Figure~\ref{fig:rearleadfollow} and $\Gamma^*=\{1\}$ for the Carla simulation. For more complex adversarial action sets and adversarial policy designs, one can refer to~\cite{capito2020modeled} for example. Note that the choice between $\Gamma$ and $\Gamma^*$ does not affect the optimality, but may exhibit significantly different computational time. To obtain the same result shown in Figure~\ref{fig:lead_follow_demo}(a), the uniform sampling in $\Gamma$ empirically takes five times the number of samples required when using $\Gamma^*$.

We further emphasize the following observations from Figure~\ref{fig:lead_follow_demo}. 
First, for all policies in both figures, as SV velocity increases and lead vehicle velocity decreases, one requires a longer following distance to ensure safety.
Second, in Figure~\ref{fig:lead_follow_demo}, the brake-to-stop policy is more conservative than the IDM, which also aligns with the observation that the cardinality, 26734, of the approximated convex hull in (a) is much larger than the cardinality, 19793, in (b)).
Finally, the dark dashed trajectories in Figure~\ref{fig:lead_follow_demo} induce behaviors that match the standard EURO NCAP~\cite{van2017euro} testing scenarios for the Automatic Emergency Braking (AEB) car-to-car system within the applicable velocity range. Note that such a concrete scenario-based testing approach is insufficient to cover the actual safe operable domain induced by the set invariance analysis nor dangerous enough to force the SV to safety-critical events. 

% Figure~\ref{fig:carla_error} gives an empirical evidence of uncertainties observed in simulating vehicle motion trajectories in Carla. 

\begin{figure}
    \centering
    \includegraphics[trim=0 4cm 0 1cm, clip, width=0.7\textwidth]{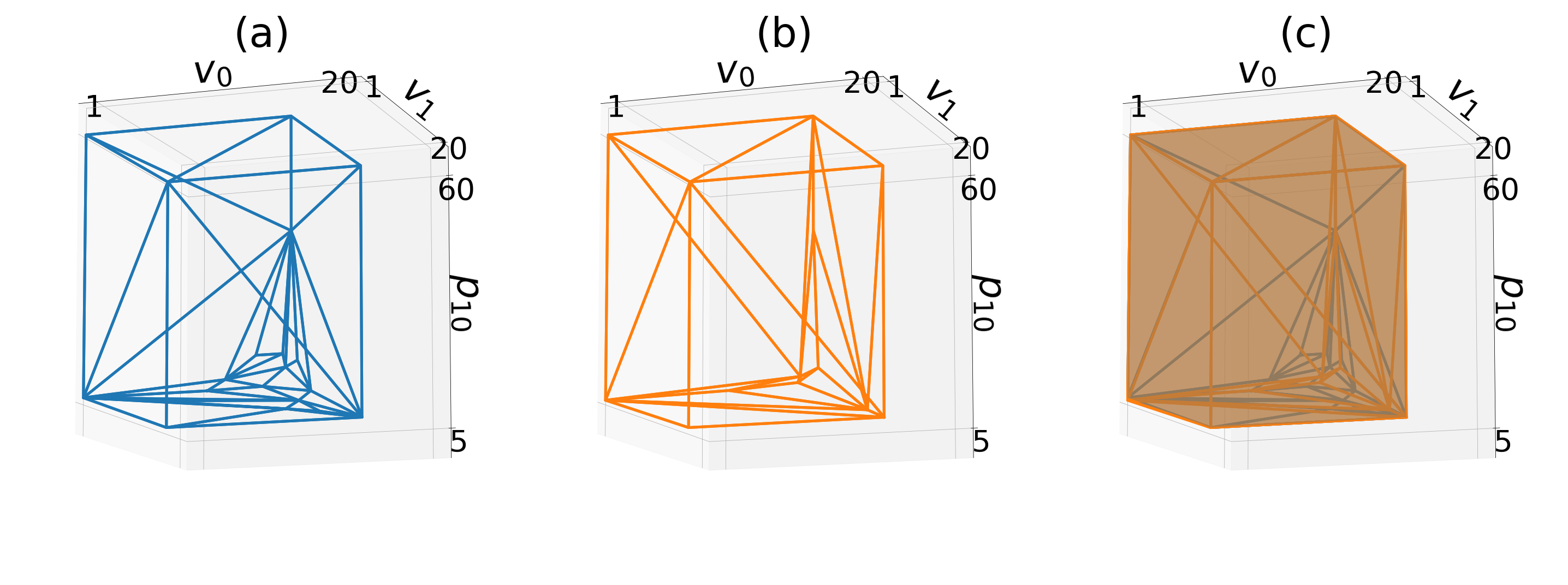}
    \caption{The maximal $\epsilon\delta$-almost robustly controlled forward invariant set ($\delta=1, \epsilon=0.01$) obtained through Algorithm~\ref{alg:qnt_spe} with confidence level of $0.9$ for two behavior agents in the Carla simulator: (a) The ``aggressive" behavior agent, (b) The ``cautious" behavior agent, (c) Combine the two agents, note that the obtained set from (b) contains the one from (a). 
    The lead-vehicle following case shares a similar configuration with Figure~\ref{fig:lead_follow_demo}. The complete process takes 1070 seconds and 1720 seconds of simulation time for the aggressive agent and the cautious agent, respectively. $\delta_t$ decays from 2.5 to 1 along the process. The initial $\Sigma$ is set to satisfy $v_0,v_1 \in [0,16](m/s)$ and $p_{10}\in[5.5,60](m)$.}
    \label{fig:carla}
\end{figure}
The lead-vehicle following case is also implemented in Carla and the quantification results for two behavior agent with different parameterized aggressiveness are shown in Figure~\ref{fig:carla}. In general, Figure~\ref{fig:carla} and Figure~\ref{fig:lead_follow_demo} share the similar observations. For example, as the SV velocity increases and as the lead vehicle's velocity decreases, one requires a longer following distance to ensure set invariance. However, the behavior agent is also different from the explicitly prameterized IDM. It is a combination of a high-level decision-making module, a low-level way-point planner, and a PID based controller. Some factors, such as the historical way-points and tracking error, are not directly captured by the state space $\Sigma$ or the action space $\mathcal{U}$, but could still affect the state-action propagation. Such unspecified factors also contribute to the disturbances and uncertainties induced by $\boldsymbol\omega$. Finally, the cautious agent is more conservative than the aggressive agent, which is also captured by Figure~\ref{fig:carla}(c) where the invariant set obtained for the cautious agent contains the one approximated for the aggressive agent.

% \begin{figure}
%     \vspace{3mm}
%     \centering
%     \includegraphics[width=0.45\textwidth]{images/cte.png}
%     \caption{Illustrating uncertainties in the Carla simulator: 10 positioning trajectories are collected from 10 runs of the same scenario with the same initialization, the same sequence of actions, and the same adjustable hyper-parameters (e.g. vehicle mass, tire friction, weather condition). Starting from the second collected trajectory, at each time step, the $\ell2$-norm positioning error w.r.t. the first collected trajectory is plotted. It is immediate that the uncertainty does exist and remains bounded empirically.}
%     \label{fig:carla_error}
% \end{figure}

\subsubsection{A fleet of three vehicles with SV in the middle}
The case studied involves three vehicles in the same lane, including the SV in the middle, a lead vehicle, and a follower. Consider the states of the SV velocity $v_0$, the lead vehicle velocity $v_1$, the follower velocity $v_2$, the relative distance between the lead vehicle and the SV $p_{10}$, and the relative distance between the SV and the follower $p_{20}$, leading to a 5-dimensional state space of $\Sigma \subseteq \R^5$. The scenario control includes the longitudinal acceleration for both the lead vehicle and the follower as  $\Gamma \subseteq \R^2$. The hard constraints are enforced similarly with the lead-vehicle following case. Correspondingly, the adversarial action set includes the follower's minimum braking capability and the maximal braking capability for the lead vehicle. Note that the admissible action space for the follower is not taking any positive acceleration control. That would induce an empty invariant set in the case where both traffic vehicles decide to cause a collision cooperatively. In practice, the construction of $\Gamma$ induces the intended aggressiveness of the scenario and should be designed appropriately such that the safe sub-domain at least exists. This is of practical value for other scenario work and is out of the scope of this paper.
\begin{figure}
    % \vspace{3mm}
    \centering
    \includegraphics[width=0.9\textwidth]{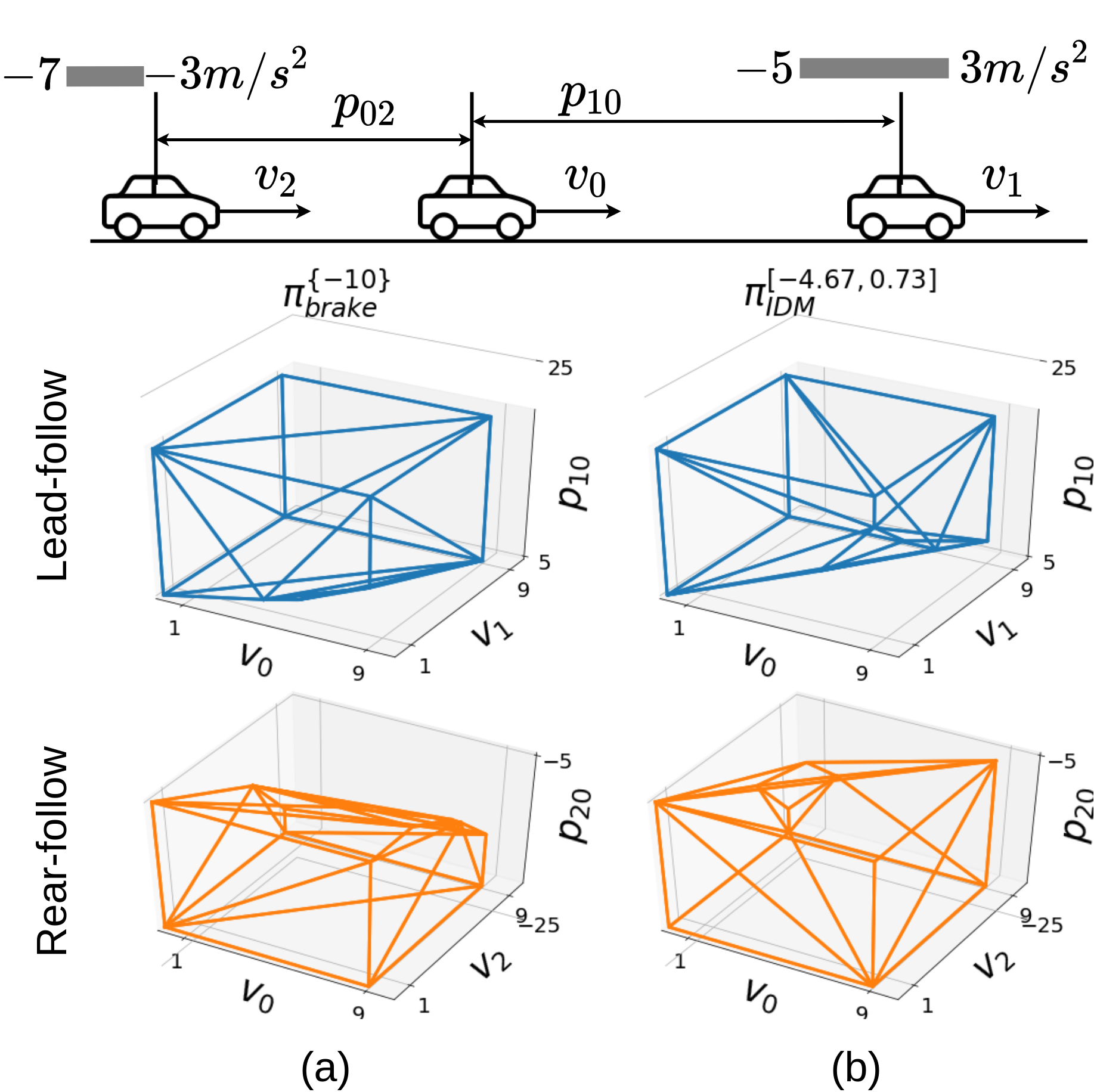}
    \caption{A synthesized example of quantifying the safety performance of two Subject Vehicle (SV) driving policies in the fleet of three vehicles with Algorithm~\ref{alg:qnt_spe}: the operable state space is described with $v_0, v_1, v_2 \in [0,6](m/s)$, $p_{10} \in [5,25](m)$, and $p_{20} \in [-25,-5](m)$. The admissible action space $\Gamma=[-5, 3]\times[-7,-3](m/s^2)$. Two underlying SV policies are evaluated, (a) the SV keeps braking-to-stop at a constant rate of $-10 m/s^2$, (b) the SV executes the parameterized Intelligent Driving Model (IDM) with maximal braking capability set to $-4.67m/s^2$. For each SV policy, the $\epsilon\delta$-almost robustly controlled forward invariant set ($\epsilon=0.001, \delta=1.5$) is quantified with a confidence level of $0.9$ through the scenario sampling approach.}
    \label{fig:rearleadfollow}
    \vspace{-3mm}
\end{figure}
\begin{table}
\vspace{3mm}
\renewcommand{\arraystretch}{1.3}
\caption{The cardinality of the approximated convex hull for each extracted subset of the $\epsilon\delta$-almost robustly controlled forward invariant set corresponding to Figure~\ref{fig:rearleadfollow}.}
\label{tab:rearleadfollow}
\centering
% Some packages, such as MDW tools, offer better commands for making tables
% than the plain LaTeX2e tabular which is used here.
\begin{tabular}{|c|c|c|c|}
        \hline
        SV Policy & Lead-follow & Rear-follow & Total \\ \hline
        $\pi_{\text{brake}}^{\{-10\}}$ & 1149.188 & 949.5 & 2098.688 \\
        $\pi_{\text{IDM}}^{[-4.67,0.73]}$ & 1054.808 & 1123.513 & 2178.321 \\
        \hline
    \end{tabular}
\end{table}

\begin{remark}
    To the best of our knowledge, the ``safe lead distance" dictated by the interactions between the SV and its longitudinal follower is rarely studied in the literature. This is mainly due to the follower bearing full responsibility for all rear-end collisions considering the commonly accepted rules and regulations. However, such regulation was primarily designed for human drivers where the forward perception capability is significantly stronger than the rear-view case. This obviously does not hold for a typical ADS/ADAS-equipped vehicle. Hence we believe the case study is of research value and potential practical impact as the traffic rules evolve to be compatible with advanced automated vehicle features.
\end{remark}

In general, visualizing the 5-dimensional state space $\Phi$ is difficult. In Figure~\ref{fig:rearleadfollow}, we extract two subsets of the obtained state space that induce the lead vehicle interactions and the rear follower interactions, respectively. The cardinality comparison is listed in Table~\ref{tab:rearleadfollow}. 

\begin{remark}
    Note that the cardinality comparison is a straightforward way of comparing safety performance, but it is not the only way. This paper is not implying that the policy with a more extensive forward invariant set is ``safer". The set quantification is a precise analysis of the SV's safety performance. Different SV policies may exhibit various properties that shape the controlled forward invariant set differently, even sharing similar cardinality. In practice, the final assessment requires and varies w.r.t. the selected metric that justifies the concept of safety~\cite{iv:2019:weng, griffor2019workshop}.
\end{remark}

We further emphasize the following observations from Figure~\ref{fig:rearleadfollow} and Table~\ref{tab:rearleadfollow}. First, for the rear follower case, as $v_0$ decreases and as $v_2$ increases, one requires a longer following distance to maintain safety. This aligns with common sense. Second, a conservative SV policy, such as $\pi_{\text{brake}}^{\{-10\}}$ may present higher threat to the follower, leading to an invariant set with small cardinality.

We conclude this section by emphasizing that while, in theory, CoD is a commonly observed problem in this field~{\cite{feng2021intelligent,zhao2017accelerated,fan2017d}}, the practical complexity of a safety evaluation algorithm still varies significantly among methods presented with the same problem. For example, the intelligent driving example in~{\cite{arief2021deep}}, the AEB system considered in~{\cite{fan2017d}}, and the studied case in this section are all essentially the same in terms of the particular ODD being studied. However, the methodology considered by~{\cite{arief2021deep}} is confined to the finite-time safety assurance with the time duration of the scenario also being part of the state space, leading to a problem of 15 dimensions with the given hyper-parameter settings. The method in~{\cite{fan2017d}} shares the same three-dimension configuration with our studied case, but the method in~{\cite{fan2017d}} is dependent upon the assumed hybrid system model restricted to a limited subset of the original ODD (e.g., only constant braking and acceleration maneuvers are considered). Hence the actual space studied by~{\cite{fan2017d}} is significantly more limited than the one considered in our work.

\section{Conclusion and Discussions} \label{sec:conclusion}
This paper proposes a novel formulation of imposing the scenario sampling safety assurance problem as a data-driven robustly controlled forward invariant set validation and quantification problem. The problem is then further divided into a pair of complete safety validation problems, the feasible safety quantification problem, and the optimal safety quantification problem. Solutions of various properties are presented, analyzed analytically, and evaluated numerically. To a certain extent, this paper is also presenting a rigorous scheme to justify future scenario sampling algorithms' performance in terms of their completeness and optimality. Such an analytical perspective complements the observed empirical success. 

Within the scope of this paper, we are able to rigorously answer the three questions raised in Section~\ref{sec:introduction}. (i) Algorithm~\ref{alg:val_delta_epsilon} validates the claim of a safe ODD in finite number of scenarios with a quantifiable confidence level and accuracy dictated by Theorem~\ref{thm:prob-complete-val}. (ii) Algorithm~\ref{alg:qnt_ae} guarantees to find at least one safe sub-domain as the number of sampled scenarios tends to infinity supported by Theorem~\ref{thm:ie-complete}. (iii) Finally, Algorithm~\ref{alg:qnt_spe} ensures the asymptotic convergence to an approximation of the optimal safe domain with arbitrarily high accuracy by Theorem~\ref{thm:spe-optimal}. 

\bibliographystyle{IEEEtran}
\bibliography{output}

\end{document}